%% file: main.tex
\documentclass[sn-mathphys,Numbered]{sn-jnl}

\usepackage{graphicx}%
\usepackage{multirow}%
\usepackage{amsmath,amssymb,amsfonts}%
\usepackage{amsthm}%
\usepackage{mathrsfs}%
\usepackage[title]{appendix}%
\usepackage[dvipsnames]{xcolor}%
\usepackage{textcomp}%
\usepackage{manyfoot}%
\usepackage{booktabs}%
\usepackage{algorithm}%
\usepackage{algorithmicx}%
\usepackage{algpseudocode}%
\usepackage{listings}%

\usepackage{mysty}

\begin{document}

\title[]{Convergence and Recovery Guarantees of Unsupervised Neural Networks for Inverse Problems}

\author*[]{\fnm{Nathan} \sur{Buskulic}}\email{nathan.buskulic@unicaen.fr}
\author[]{\fnm{Jalal} \sur{Fadili}}\email{Jalal.Fadili@ensicaen.fr}
\author[]{\fnm{Yvain} \sur{Qu\'eau}}\email{yvain.queau@ensicaen.fr}

\affil[]{\orgdiv{Greyc}, \orgname{Normandie Univ., UNICAEN, ENSICAEN, CNRS}, \orgaddress{\street{6 Boulevard Maréchal Juin}, \city{Caen}, \postcode{14000}, \country{France}}}

\abstract{
Neural networks have become a prominent approach to solve inverse problems in recent years. While a plethora of such methods was developed to solve inverse problems empirically, we are still lacking clear theoretical guarantees for these methods. On the other hand, many works proved convergence to optimal solutions of neural networks in a more general setting using overparametrization as a way to control the Neural Tangent Kernel. In this work we investigate how to bridge these two worlds and we provide deterministic convergence and recovery guarantees for the class of unsupervised feedforward multilayer neural networks trained to solve inverse problems. We also derive overparametrization bounds under which a two-layers Deep Inverse Prior network with smooth activation function will benefit from our guarantees.}



\maketitle

\input{tex/sec_Intro}
\input{tex/sec_Prelim}
\input{tex/sec_Mainresults}

\input{tex/sec_DIP}

\input{tex/sec_Numerical}



\bibliography{references}

\appendix

\input{tex/appendix}

\end{document}

%% file: tex/sec_Intro.tex
\section{Introduction}
%

\subsection{Problem Statement}
An inverse problem consists in reliably recovering a signal $\xvc \in \R^n$ from noisy indirect observations
\begin{align}\label{eq:forward}
\yv = \fopnl(\xvc) + \veps ,
\end{align}
where $\yv \in \R^m$ is the observation, $\fopnl: \R^n \to \R^m$ is a forward operator, and $\veps$ stands for some additive noise. We will denote by $\yvc = \fopnl(\xvc)$ the ideal observations i.e., those obtained in the absence of noise. 

In recent years, the use of sophisticated machine learning algorithms, including deep learning, to solve inverse problems has gained a lot of momentum and provides promising results; see e.g., the reviews \cite{arridge_solving_2019,ongie_deep_2020}. The general framework of these methods is to optimize a generator network $\mathbf{g}: (\uv,\thetav) \in \R^d\times \R^p \mapsto \xv \in \R^n$, with some activation function~$\phi$, to transform a given input $\uv \in \R^d$ into a vector $\xv \in \R^n$. The parameters $\thetav$ of the network are optimized via (possibly stochastic) gradient descent to minimize a loss function $\lossy:\funspacedef{\R^m}{\R_+}, \yvt \mapsto \lossy(\yvt)$ which measures the discrepancy between the observation $\yv$ and the solution $\yvt = \fopnl(\gdipt)$ generated by the network at time $t \geq 0$.

Theoretical understanding of recovery and convergence guarantees for deep learning-based methods is of paramount importance to make their routine use in critical applications reliable \cite{mukherjee2023learned}. While there is a considerable amount of work on the understanding of optimization dynamics of neural network training, especially through the lens of overparametrization, recovery guarantees when using neural networks for inverse problem remains elusive. Some attempts have been made in that direction but they are usually restricted to very specific settings. One kind of results that was obtained~\cite{li_nett_2020,mukherjee2020learned,schwab2019deep} is convergence towards the optimal points of a regularized problem, typically with a learned regularizer. However this does not give guarantees about the real sought-after vector. Another approach is used in Plug-and-Play~\cite{liu2021recovery} to show that under strong assumptions on the pre-trained denoiser, one can prove convergence to the true vector. This work is however limited by the constraints on the denoiser which are not met in many settings.

Our aim in this paper is to help close this gap by explaining when gradient descent consistently and provably finds global minima of $\loss$, and how this translates into recovery guarantees for both $\yvc$ and $\xvc$ i.e., in both the observation and the signal spaces. For this, we focus on a continuous-time gradient flow applied to $\loss$:
\begin{align}
\begin{cases}
\dot{\thetav}(t) = - \nabla_{\thetav} \lossy(\fopnl(\gdipt))  \\
\thetav(0) = \thetav_0 .
\end{cases}\label{eq:gradflow}
\end{align}
This is an idealistic setting which makes the presentation simpler and it is expected to reflect the behavior of practical and common first-order descent algorithms, as they are known to approximate gradient flows.

In this work, our focus in on an unsupervised method known as Deep Image Prior~\cite{ulyanov_deep_2020}, that we also coin Deep Inverse Prior (DIP) as it is not confined to images. A chief advantage of this method is that it does not need any training data, while the latter is mandatory in most supervised deep learning-based methods used in the literature. In the DIP method, $\uv$ is fixed throughout the optimization/training process, usually a realization of a random variable. By taking out the need of training data, this method focuses on the generation capabilities of the network trained through gradient descent. In turn, this will allow us to get insight into the effect of network architecture on the reconstruction quality.

\subsection{Contributions}
We deliver a theoretical analysis of gradient flow optimization of neural networks, i.e. \eqref{eq:gradflow}, in the context of inverse problems and provide various recovery guarantees for general loss functions verifying the Kurdyka-\L ojasewicz (KL) property. We first prove that the trained network with a properly initialized gradient flow will converge to an optimal solution in the observation space with a rate characterized by the desingularizing function appearing in the KL property of the loss function. This result is then converted to a prediction error on $\yvc$ through an early stopping strategy. More importantly, we present a recovery result in the signal space with an upper bound on the reconstruction error of $\xvc$. The latter result involves for instance a restricted injectivity condition on the forward operator.  

We then turn to showing how these results can be applied to the case of a two-layer neural network in the DIP setting where
\begin{align}\label{eq:dipntk}
    \mathbf{g}(\uv,\thetav) = \frac{1}{\sqrt{k}}\Vv \phi(\Wv\uv), \quad \thetav \eqdef (\Vv,\Wv) ,
\end{align}
with  $\Vv \in \R^{n \times k}$, $\Wv \times \R^{k \times d}$, and $\phi$ an element-wise nonlinear activation function. The scaling by $\sqrt{k}$ will become clearer later. We show that for a proper random initialization $\Wv(0)$, $\Vv(0)$ and sufficient overparametrization, all our conditions are in force to control the eigenspace of the Jacobian of the network as required to obtain the aforementioned convergence properties. We provide a characterization of the overparametrization needed in terms of $(k,d,n)$ and the conditioning of $\fopnl$.

\subsection{Relation to Prior Work}\label{sec:prior}


\paragraph{\textbf{Data-Driven Methods to Solve Inverse Problems}}
Data-driven approaches to solve inverse problems come in various forms; see the comprehensive reviews in~\cite{arridge_solving_2019,ongie_deep_2020}. The first type trains an end-to-end network to directly map the observations to the signals for a specific problem. 
While they can provide impressive results, these methods can prove very unstable as they do not use the physics of the problem which can be severely ill-posed. To cope with these problems, several hybrid models that mix model- and data-driven algorithms were developed in various ways. One can learn the regularizer of a variational problem  \cite{prost_learning_2021} or use Plug-and-Play methods \cite{venkatakrishnan_plug-and-play_2013} for example. Another family of approaches, which takes inspiration from classical iterative optimization algorithms, is based on unrolling 
(see \cite{monga_algorithm_2021} for a review of these methods). 
Still, all these methods require an extensive amount of training data, which may not always be available. 

\paragraph{\textbf{Deep Inverse Prior}}
The DIP model \cite{ulyanov_deep_2020} (and its extensions that mitigate some of its empirical issues ~\cite{liu_image_2019,mataev_deepred_2019,shi_measuring_2022,zukerman_bp-dip_2021}) is an unsupervised alternative to the supervised approches briefly reviewed above. The empirical idea is that the architecture of the network acts as an implicit regularizer and will learn a more meaningful transformation before overfitting to artefacts or noise. With an early stopping strategy, one can hope for the network to generate a vector close to the sought-after signal. However, this remains purely empirical and there is no guarantee that a network trained in such manner converges in the observation space (and even less in the signal space). The theoretical recovery guarantees of these methods are not well understood \cite{mukherjee2023learned} and our work aims at reducing this theoretical gap by analyzing the behaviour of such networks in both the observation and the signal space under some overparametrization condition.

\paragraph{\textbf{Theory of Overparametrized Networks}}
To construct our analysis, we build upon previous theoretical work of overparametrized networks and their optimization trajectories~\cite{bartlett_deep_2021,fang_mathematical_2021}. The first works that proved convergence to an optimal solution were based on a strong convexity assumption of the loss which is typically not the case when it is composed with a neural network. A more recent approach is based on a gradient dominated inequality from which we can deduce by simple integration an exponential convergence of the gradient flow to a zero-loss solution. This allows to obtain convergence guarantees for networks trained to minimize a mean square error by gradient flow~\cite{chizat_lazy_2019} or its discrete counterpart (i.e., gradient descent with fixed step)~\cite{du_gradient_2019,arora_fine-grained_2019, oymak_overparameterized_2019, oymak_toward_2020}. The work that we present here is inspired by these works but it goes far beyond them. Amongst other differences, we are interested in the challenging situation of inverse problems (presence of a forward operator), and we deal with more general loss functions that obey the Kurdyka-\L ojasewicz inequality (e.g., any semi-algebraic function or even definable on an o-minimal structure)~\cite{loj1,loj3,kurdyka_gradients_1998}. 


Recently, it has been found that some kernels play a very important role in the analysis of convergence of the gradient flow when used to train neural networks. In particular the semi-positive definite kernel given by $\Jgt\Jgt\tp$, where $\Jgt$ is the Jacobian of the network at time $t$. When all the layers of a network are trained, this kernel is a combination of the \textit{Neural Tangent Kernel} (NTK)~\cite{jacot_neural_2018} and the Random Features Kernel (RF)~\cite{rahimi_weighted_2008}. If one decides to fix the last layer of the network, then this amounts to just looking at the NTK which is what most of the previously cited works do. The goal is then to control the eigenvalues of the kernel to ensure that it stays positive definite during training, which entails convergence to a zero-loss solution at an exponential rate. The control of the eigenvalues of the kernel is done through a random initialization and the overparametrization of the network. Indeed, for a sufficiently wide network, the parameters $\thetavt$ will stay near their initialization and they will be well approximated by their linearization (so-called ``lazy'' regime~\cite{chizat_lazy_2019}). The overparametrization bounds that were obtained are mostly for two-layers networks as the control of deep networks is much more complex.


However, even if there are theoretical works on the gradient flow-based optimization of neural networks as reviewed above, similar analysis that would accommodate for the forward operator as in inverse problems remain challenging and open. Our aim is to participate in this endeavour by providing theoretical understanding of recovery guarantees with neural network-based methods. 

This paper is an extension of our previous one in~\cite{buskulic2023convergence}. There are however several distinctive and new results in the present work. For instance, the work \cite{buskulic2023convergence} only dealt with linear inverse problems while our results here apply to non-linear ones. Moreover, we here provide a much more general analysis under which we obtain convergence guarantees for a wider class of models than just the DIP one and for a general class of loss functions, not just the MSE. More importantly we show convergence not only in the observation space but also in the signal space now. When particularized to the DIP case, we also provide overparametrization bounds for the case when the linear layer of the network is not fixed which is also an additional novelty.

\paragraph{Paper organization}
The rest of this work is organized as follows. In Section~\ref{sec:prelim} we give the necessary notations and definitions useful for this work. In Section~\ref{sec:main_res} we present our main result with the associated assumptions and proof. In Section~\ref{sec:dip} we present the overparametrization bound on the DIP model. Finally, in Section~\ref{sec:expes}, we show some numerical experiments that validate our findings, before drawing our conclusions in Section~\ref{sec:conclu}.

%% file: tex/sec_Prelim.tex
\section{Preliminaries}\label{sec:prelim}
\subsection{General Notations}
For a matrix $\Mv \in \R^{a \times b}$ we denote by $\sigmin(\Mv)$ and $\sigmax(\Mv)$ its smallest and largest non-zero singular values, and by $\kappa(\Mv) = \frac{\sigmax(\Mv)}{\sigmin(\Mv)}$ its condition number. We also denote by $\dotprod{}{}$ the Euclidean scalar product, $\norm{\cdot}$ the associated norm (the dimension is implicit from the context), and $\normf{\cdot}$ the Frobenius norm of a matrix. With a slight abuse of notation $\norm{\cdot}$ will also denote the spectral norm of a matrix. We use $\Mv^i$ (resp. $\Mv_i$) as the $i$-th row (resp. column) of $\Mv$. For two vectors $\xv,\zv$, $[\xv,\zv]=\enscond{(1-\rho)\xv+\rho\zv}{\rho \in [0,1]}$ is the closed segment joining them. We use the notation $a \gtrsim b$ if there exists a constant $C > 0$ such that $a \geq C b$.

We also define $\yvt = \fopnl(\gdipt)$ and $\xv(t) = \gdipt$ and we recall $\yvc = \fopnl(\xvc)$. The Jacobian of the network is denoted $\Jg$. $\Jgt$ is a shorthand  notation of $\Jg$ evaluated at $\thetav(t)$. $\Jft$ is the Jacobian of the forward operator $\fopnl$ evaluated at $\xv(t)$. The local Lipschitz constant of a mapping on a ball of radius $R > 0$ around a point $\zv$ is denoted $\Lip_{\Ball(\zv,R)}(\cdot)$. We omit $R$ in the notation when the Lipschitz constant is global. For a function $f: \R^n \to \R$, we use the notation for the sublevel set $[f < c] = \enscond{\zv \in \R^n}{f(\zv) < c}$ and $[c_1 < f < c_2] = \enscond{\zv \in \R^n}{ c_1 < f(\zv) < c_2}$.

Given $\zv \in \cC^0(]0,+\infty[;\R^a)$, the set of cluster points of $\zv$ is defined as
\[
\cluster{\zv(\cdot)} = \enscond{\widetilde{\zv} \in \R^a}{\exists (t_k)_{k\in\N} \to +\infty \tstt \lim_{k \to \infty} \zv(t_k) = \widetilde{\zv}} .
\]

For some $\Theta \subset \R^p$, we define $\Sigma_\Theta = \enscond{\gv(\uv,\thetav)}{\thetav \in \Theta}$ the set of signals that the network $\gv$ can generate for all $\theta$ in the parameter set $\Theta$. $\Sigma_\Theta$ can thus be viewed as a parametric manifold. If $\Theta$ is closed (resp. compact), so is $\Sigma_\Theta$. We denote $\dist(\cdot,\Sigma_\Theta)$ the distance to $\Sigma_\Theta$ which is well defined if $\Theta$ is closed and non-empty. For a vector $\xv$, $\xvsigmatheta$ is its projection on $\Sigma_\Theta$, i.e. $\xvsigmatheta \in \Argmin_{\zv \in \Sigma_\Theta} \norm{\xv-\zv}$. Observe that $\xvsigmatheta$ always exists but might not be unique. We also define $T_{\Sigma_\Theta}(\xv) = \conv{\R_+(\Sigma_\Theta-\xv)}$ the tangent cone of $\Sigma_\Theta$ at $\xv\in\Sigma_\Theta$.

The minimal (conic) singular value of a matrix $\fop \in \R^{m \times n}$ w.r.t. the cone $T_{\Sigma_\Theta}(\xv)$ is then defined as

\[
\lmin(\fop;T_{\Sigma_\Theta}(\xv)) = \inf
\{\norm{\fop \zv}/\norm{\zv}:  \zv \in T_{\Sigma_\Theta}(\xv)\}.
\]


\subsection{Multilayer Neural Networks}
Neural networks produce structured parametric families of functions that have been studied and used for almost 70 years, going back to the late 1950's~\cite{rosenblatt1958perceptron}.
\begin{definition}\label{def:nn}
Let $d,L\in \N$ and $\phi : \R \to \R$ an activation map which acts componentwise on the entries of a vector. A fully connected multilayer neural network with input dimension $d$, $L$ layers and activation $\phi$, is a collection of weight matrices $\pa{\Wv^{(l)}}_{l \in [L]}$ and bias vectors $\pa{\bv^{(l)}}_{l \in [L]}$, where $\Wv^{(l)} \in \R^{N_l\times N_{l-1}}$ and $\bv^{(l)} \in \R^{N_l}$, with $N_0=d$, and $N_l \in \N$ is the number of neurons for layer $l \in [L]$. Let us gather these parameters as
	\[
	\thetav=\pa{(\Wv^{(1)},\bv^{(1)}), \ldots, (\Wv^{(L)},\bv^{(L)})} \in \bigtimes_{l=1}^L \pa{\pa{\R^{N_l \times N_{l-1}}} \times \R^{N_l}}.
	\]
Then, a neural network parametrized by $\thetav$ produces a function   
\begin{align*}
\gv: & ~ (\uv,\thetav) \in \R^d \times \bigtimes_{l=1}^L \pa{\pa{\R^{N_l \times N_{l-1}}} \times \R^{N_l}} \mapsto \gv(\uv,\thetav) \in \R^{N_L} , \qwithq N_L = n ,
	\end{align*}
	which can be defined recursively as 
	\begin{align*}
		\begin{cases}
			\gv^{(0)}(\uv,\thetav)&= \uv,
			\\
			\gv^{(l)}(\uv,\thetav)&=\phi\pa{\Wv^{(l)}\gv^{(l-1)}(\uv,\thetav)+\bv^{(l)}}, \quad \text{ for } l=1,\ldots , L-1,
			\\
			\gv(\uv,\thetav)&= \Wv^{(L)} \gv^{(L-1)}(\uv,\thetav) + \bv^{(L)} .
		\end{cases}
	\end{align*}
\end{definition}
The total number of parameters is then $p = \sum_{l=1}^L (N_{l-1}+1)N_l$. In the rest of this work, $\gdip$ is always defined as just described. We will start by studying the general case before turning in Section~\ref{sec:dip} to a two-layer network, i.e. with $L=2$.

\subsection{KL Functions}\label{subsec:kl}
We will work with a general class of loss functions $\loss$ that are not necessarily convex. More precisely, we will suppose that $\loss$ verifies a Kurdyka-\L ojasewicz-type (KL for short) inequality.
\begin{definition}[KL inequality]\label{def:KL}
A continuously differentiable function $f:\funspacedef{\R^n}{\R}$ satisfies the KL inequality if there exists $r_0 > 0$ and a strictly increasing function $\psi \in \cC^0([0,r_0[) \cap \cC^1(]0,r_0[)$ with $\psi(0) = 0$ such that
\begin{align}\label{eq:KLpsi}
\psi'(f(\zv)-\min f)\norm{\nabla f(\zv)} \geq 1, \qforallq \zv \in [\min f < f < \min f + r_0] . 
\end{align}
We use the shorthand notation $f \in \KLpsi(r_0)$ for a function satisfying this inequality. 
\end{definition}
The KL property basically expresses the fact that the function $f$ is sharp under a reparameterization of its
values. Functions satisfying the KL inequality are also sometimes called gradient dominated functions~\cite{nesterov2006cubic}. The function $\psi$ is known as the desingularizing function for~$f$. The {\L}ojasiewicz inequality~\cite{loj1,loj3} corresponds to the case where the desingularizing function takes the form $\psi(s) = cs^\alpha$ with $\alpha\in[0,1]$. The KL inequality plays a fundamental role in several fields of applied mathematics among which convergence behaviour of  (sub-)gradient-like systems and minimization algorithms~\cite{AbsilAnalytic05,HuangKL06,BolteKL07,AttouchAnalytic09,attouch2010proximal,bolte2010characterizations}, neural networks~\cite{QuincampoixNNKL06}, partial differential equations~\cite{SimonKL83,HarauxKL01,ChillKL06}, to cite a few. The KL inequality is closely related to error bounds that also play a key role to derive complexity bounds of gradient descent-like algorithms~\cite{bolte2017error}.

\modifs{Our KL definition is somehow globalized with regards to the original one~\cite{kurdyka_gradients_1998} as we require the inequality to hold at any point in the sublevel set $[\min f < f < \min f + r_0]$ without intersecting the latter with a neighborhood of the minimizers. Of course, if $f$ is sublevel set is bounded, then the above inequality is automatically localize. On the other hand, we require the KL property to hold only at global minimizers and not at any critical point. Nevertheless, in general, our globalized KL inequality entails that the function cannot have critical points that are not global minimizers. However, since we impose this assumption only on the loss function $\loss$, this is not a restrictive assumption that match many usual loss functions. Let us give some examples of functions satisfying \eqref{eq:KLpsi}.}
\begin{example}[Convex functions with sufficient growth]
Let $f$ be a differentiable convex function on $\R^n$ such that $\Argmin(f) \neq \emptyset$. Assume that $f$ verifies the growth condition
\begin{equation}\label{eq:exgrowth}
f(\zv) \geq \min f + \varphi(\dist(\zv,\Argmin(f))), \qforallq \zv \in [\min f < f < \min f + r] ,
\end{equation}
where $\varphi: \R_+ \to \R_+$ is continuous, increasing, $\varphi(0)=0$ and $\int_0^{r} \frac{\varphi^{-1}(s)}{s} ds < +\infty$. Then by \cite[Theorem~30]{bolte2010characterizations}, $f \in \KLpsi(r)$ with $\psi(r)=\int_0^{r} \frac{\varphi^{-1}(s)}{s} ds$.
\end{example}

\begin{example}[Uniformly convex functions]
Suppose that $f$ is a differentiable uniformly convex function, i.e., $\forall \zv,\xv \in \R^n$,
\begin{equation}\label{eq:exunifconv}
f(\xv) \geq f(\zv) + \dotprod{\nabla f(\zv)}{\xv-\zv} + \varphi\pa{\norm{\xv-\zv}}
\end{equation}
for an increasing function $\varphi: \R_+ \to \R_+$ that vanishes only at $0$. Thus $f$ has a unique minimizer, say $\zv^\star$, see~\cite[Proposition~17.26]{BauschkeBook}. This example can then be deduced from the previous one since a uniformly convex function obviously obeys \eqref{eq:exgrowth}. However, we here provide an alternative and sharper characterization. We may assume without loss of generality that $\min f = 0$. Applying inequality \eqref{eq:exunifconv} at $\xv = \zv^\star$ and any $\zv \in [0 < f]$, we get
\begin{align*}
f(\zv) 
&\leq \dotprod{\nabla f(\zv)}{\zv-\xv} - \varphi\pa{\norm{\xv-\zv}} \\
&\leq \norm{\nabla f(\zv)}\norm{\xv-\zv} - \varphi\pa{\norm{\xv-\zv}} \\
&\leq \varphi_+(\norm{\nabla f(\zv)}) ,
\end{align*}
where $\varphi_+: a \in \R_+ \mapsto \varphi^+(a)=\sup_{x \geq 0} a x - \varphi(x)$ is known as the monotone conjugate of $\varphi$. $\varphi_+$ is a proper closed convex and non-decreasing function on $\R_+$ that vanishes at 0. When $\varphi$ is strictly convex and supercoercive, so is $\varphi_+$ which implies that $\varphi_+$ is also strictly increasing on $\R_+$. Thus $f$ verifies Definition~\ref{def:KL} at any $\zv \in [0 < f]$ with $\psi$ a primitive of $\frac{1}{\varphi_+^{-1}}$, and $\psi$ is indeed strictly increasing, vanishes at $0$ and is even concave.  A prominent example is the case where $\varphi: s \in \R_+ \mapsto \frac{1}{p}s^p$, for $p \in ]1,+\infty[$, in which case $\psi: s \in \R_+ \mapsto q^{-1/q} s^{1/p}$, where $1/p+1/q=1$. 
\end{example}

\begin{example}
\modifs{
For the original KL inequality \cite{kurdyka_gradients_1998}, deep results from algebraic geometry have shown that in finite-dimensional spaces, the KL inequality is satisfied by a large class of functions, namely, real semi-algebraic functions and more generally, function definable on an o-minimal structure or even functions belonging to analytic-geometric categories \cite{coste2000introduction,van1998tame,kurdyka_gradients_1998,loj1,loj3}. Definable convex functions do satisfy our globalized KL inequality. Note that even smooth coercive convex functions do not necessarily satisfy the KL inequality; see the counterexample in \cite[Section~4.3]{bolte2010characterizations}. Fortunately, many popular losses used in machine learning and signal processing turn out to satisfy our globalized KL inequality since the lack of local minimizers is a desirable property for such losses (MSE, Kullback-Leibler divergence and cross-entropy to cite a few).
}
\end{example}



%% file: tex/sec_Mainresults.tex
\section{Recovery Guarantees}\label{sec:main_res}

\subsection{Main Assumptions} 
Throughout this paper, we will work under the following standing assumptions.
\vspace*{2em}
\begin{mdframed}[frametitle={Assumptions on the loss}]
    \begin{assumption}\label{ass:l_smooth}
        $\lossy(\cdot) \in \cC^1(\R^m)$ whose gradient is Lipschitz continuous on the bounded sets of~$\R^m$. 
    \end{assumption}
    \begin{assumption}\label{ass:l_kl}
        $\lossy(\cdot) \in\KLpsi(\lossy(\yvz)+\eta)$ for some $\eta > 0$. 
    \end{assumption}
    \begin{assumption}\label{ass:min_l_zero}
        $\min \lossy(\cdot) = 0$.
    \end{assumption}
    \begin{assumption}\label{ass:nablaL_F}
    $\exists \Theta \subset \R^p$ with large enough diameter such that $\nabla_{\vv}\lossy(\vv) \in \ran{\Jf(\xv)}$ for any $\vv=\fopnl(\xv)$ with $\xv \in \Sigma_\Theta$. 
    \end{assumption}
\end{mdframed}

\begin{mdframed}[frametitle={Assumption on the activation}]
        \begin{assumption}\label{ass:phi_diff}
        $\phi \in \cC^1(\R)$ and $\exists B > 0$ such that $\sup_{x \in \R}|\phi'(x)| \leq B$ and $\phi'$ is $B$-Lipschitz continuous.
    \end{assumption}
\end{mdframed}

\begin{mdframed}[frametitle={Assumption on the forward operator}]
    \begin{assumption}\label{ass:F_diff}
        $\fopnl \in \cC^1(\R^n;\R^m)$ whose Jacobian $\Jf$ is Lipschitz continuous on the bounded sets of~$\R^n$.
    \end{assumption}
\end{mdframed}

\medskip


Let us now discuss the meaning and effects of these assumptions. First, \ref{ass:l_smooth} is made for simplicity to ensure existence and uniqueness of a strong maximal solution (in fact even global thanks to our estimates) of \eqref{eq:gradflow} thanks to the Cauchy-Lipschitz theorem (see hereafter). We think this could be relaxed to cover non-smooth losses if we assume path differentiability, hence existence of an absolutely continuous trajectory. This is left to a future work. A notable point in~\ref{ass:l_kl} is that convexity is not always needed for the loss (see the statements of the theorem). Regarding~\ref{ass:min_l_zero}, it is natural yet it would be straightforward to relax it. 

Assumption \ref{ass:nablaL_F} allows us to leverage the fact that
\begin{equation}\label{eq:sigmaF}
\sigminF \eqdef \inf_{\xv \in \Sigma_\Theta, \zv \in \ran{\Jf(\xv)}} \frac{\norm{\Jf(\xv)\tp\zv}}{\norm{\zv}} > 0,
\end{equation}
with $\Theta$ a sufficiently large subset of parameters. Clearly, we will show later that the parameter trajectory $\thetav(t)$ is contained in a ball around $\thetavz$. Thus a natural choice of $\Theta$ is that ball (or an enlargement of it).

There are several scenarios of interest where assumption~\ref{ass:nablaL_F} is verified.
This is the case when $\fopnl$ is an immersion, which implies that $\Jf(\xv)$ is surjective for all $\xv$. Other interesting cases are when $\lossy(\vv)=\eta\pa{\norm{\vv-\yv}^2}$, $\fopnl = \Phi \circ \fop$, where $\eta: \R_+ \to \R_+$ is differentiable and vanishes only at $0$, and $\Phi: \R^m \to \R^m$ is an immersion\footnote{Typical cases of practical interest are linear inverse problems ($\Phi$ identity) or  phase retrieval ($\Phi = |\cdot|^2$).}. One easily sees in this case that $\nabla_{\vv}\lossy(\vv) = 2\eta'\pa{\norm{\vv-\yv}^2}(\vv-\yv)$ with $\vv=\Phi(\fop\xv)$, and $\Jf(\xv)=\mathcal{J}_{\Phi}(\fop\xv)\fop$. It is then sufficient to require that $\fop$ is surjective. This can be weakened for the linear case, i.e. $\Phi$ is the identity, in which case it is sufficient that $\yv \in \ran{\fop}$ for \ref{ass:nablaL_F} to hold.

Assumption \ref{ass:phi_diff} is key in well-posedness as it ensures, by Definition~\ref{def:nn} which $\gv(\uv,\thetav)$ follows, that $\gv(\uv,\cdot)$ is $\cC^1(\R^p;\R^p)$ whose Jacobian is Lipschitz continuous on bounded sets, which is necessary for the Cauchy-Lipschitz theorem. This constraint on $\phi$ is met by many activations such as the softmax, sigmoid or hyperbolic tangent. Including the ReLU requires more technicalities that will be avoided here.

Finally, Assumption~\ref{ass:F_diff} on local Lipschitz continuity on $\fopnl$ is not only important for well-posedness of \eqref{eq:gradflow}, but it turns out to be instrumental when deriving recovery rates (as a function of the noise) in the literature of regularized nonlinear inverse problems; see \cite{ScherzerBook09} and references therein.

\subsection{Well-posedness}\label{subsec:welllocal} 
In order for our analysis to hold, the Cauchy problem~\eqref{eq:gradflow} needs to be well-posed. We start by showing that~\eqref{eq:gradflow} has a unique maximal solution. 

\begin{proposition}\label{prop:welllocal}
Assume that \ref{ass:l_smooth}, \ref{ass:phi_diff} and \ref{ass:F_diff} hold. There there exists $T(\thetav_0) \in ]0,+\infty]$ and a unique maximal solution $\thetav(\cdot) \in \cC^0([0,T(\thetav_0)[)$ of \eqref{eq:gradflow}, and $\thetav(\cdot)$ is $\cC^1$ on every compact set of the interior of $[0,T(\thetav_0)[$.
\end{proposition}

\begin{proof}
Thanks to \ref{ass:phi_diff}, one can verify with standard differential calculus applied to $\gv(\uv,\cdot)$, as given in Definition~\ref{def:nn}, that $\Jg$ is Lipschitz continuous on the bounded sets of $\R^p$. This together with \ref{ass:l_smooth} and \ref{ass:F_diff} entails that $\nabla_{\thetav} \lossy(\fopnl(\gv(\uv,\cdot))$ is also Lipschitz continuous on the bounded sets of $\R^p$. The claim is then a consequence of the Cauchy-Lipschitz theorem \cite[Theorem~0.4.1]{Haraux91}.
\end{proof}

$T(\thetav_0)$ is known as the maximal existence time of the solution and verifies the alternative: either $T(\thetav_0)=+\infty$  and the solution is called {\textit{global}};  or $T(\thetav_0)<+\infty$ and the solution blows-up in finite time, i.e., $\norm{\thetav(t)} \to +\infty$ as $t \to  T(\thetav_0)$. We will show later that the maximal solution of \eqref{eq:gradflow} is indeed global; see Section~\ref{subsec:wellglobal}.

\subsection{Main Results} 
We are now in position to state our main recovery result.

\begin{theorem}\label{thm:main}
Recall $\sigminF$ from \eqref{eq:sigmaF}. Consider a network $\gv(\uv,\cdot)$, a forward operator $\fopnl$ and a loss $\loss$, such that~\cref{ass:l_smooth} to \ref{ass:F_diff} hold. Let $\thetav(\cdot)$ be a solution trajectory of \eqref{eq:gradflow} where the initialization $\thetavz$ is such that
\begin{equation}\label{eq:bndR}
\sigminjgz > 0 \qandq R' < R
\end{equation}
where $R'$ and $R$ obey
\begin{equation}\label{eq:RandR'}
R' = \frac{2}{\sigminF\sigminjgz}\psi(\lossyzy) \qandq R = \frac{\sigminjgz}{2\Lip_{\Ball(\thetavz,R)}(\Jg)} . 
\end{equation}
Then the following holds:
\begin{enumerate}[label=(\roman*)]
\item \label{thm:main_y_bounded} the loss converges to $0$ at the rate
\begin{align}\label{eq:lossrate}
    \lossyty &\leq \Psi\inv(\gamma(t))
\end{align}
with $\Psi$ a primitive of $-\psi'^2$ and $\gamma(t) = \frac{\sigminF^2\sigminjgz^2}{4}t + \Psi\left(\lossyzy\right)$. Moreover, $\thetav(t)$ converges to a global minimizer $\thetav_{\infty}$ of $\lossy(\fopnl(\gv(\uv,\cdot)))$, at the rate
\begin{align}\label{eq:thetarate}
\norm{\thetav(t) - \thetav_{\infty}} &\leq \frac{2}{\sigminjgz\sigminF}\psi\pa{\Psi\inv\pa{\gamma(t)}} .
\end{align}


\item If $\Argmin(\lossy(\cdot)) = \ens{\yv}$, then $\lim_{t \to +\infty}\yv(t)=\yv$. In addition, if $\loss$ is convex then
\begin{align}\label{eq:yrate}
\norm{\yvt - \yvc} \leq 2\norm{\veps} \quad \text{when} \quad t\geq \frac{4\Psi(\psi\inv(\norm{\veps}))}{\sigminF^2\sigminjgz^2}-\Psi(\lossy(\yvz)).
\end{align}

\item Assume that $\Argmin(\lossy(\cdot)) = \ens{\yv}$, $\loss$ is convex, and that\footnote{We suppose here that $\Argmin_{\zv \in \Sigma} \norm{\zv-\xvc} = \{\xvcsigma\}$ is a singleton. In fact, we only need that there exists at least one $\xvcsigma \in \Argmin_{\zv \in \Sigma} \norm{\zv-\xvc}$ such that $\muf > 0$.}
\begin{assumption}\label{ass:F_inj}
$\muf > 0$ where
$
\displaystyle{\muf \eqdef \inf_{\xv \in \Sigma'} \frac{\norm{\fopnl(\xv)-\fopnl(\xvcsigma)}}{\norm{\xv-\xvcsigma}} \text{ with } \Sigma' \eqdef \Sigma_{\Ball_{R'+\norm{\thetavz}}(0)}}.
$
\end{assumption}
Let $\LF \eqdef \max_{\xv \in \Ball(0,2\norm{\xvc})} \norm{\Jf(\xv)} < +\infty$. Then
\begin{align}\label{eq:xrate}
\begin{split}
\norm{\xvt - \xvc} 	&\leq \frac{2\psi\pa{\Psi\inv\pa{\gamma(t)}}}{\muf\sigminjgz\sigminF} + \pa{1+\frac{\LF}{\muf}}\dist(\xvc,\Sigma') + \frac{\norm{\veps}}{\muf}.
\end{split}
\end{align}
\end{enumerate}
\end{theorem}

\begin{proof}
\modifs{See Section~\ref{sec:proof}.}
\end{proof}

\subsection{Discussion and Consequences}


We first discuss the meaning of the initialization condition $R' < R$. This dictates that $\psi(\lossy(\yvz))$ must be smaller than some constant that depends on the operator $\fopnl$ and the Jacobian of the network at initialization. Intuitively, this requires the initialization of the network to be in an appropriate convergence basin i.e., we start close enough from an optimal solution.

\subsubsection{Convergence Rate} The first result ensures that under the conditions of the theorem, the network converges towards a zero-loss solution. The convergence speed is given by the application of $\Psi\inv$, which is (strictly) decreasing by definition, on an affine function w.r.t time. The function $\Psi$ only depends on the chosen loss function and its associated Kurdyka-\L ojasewiecz inequality. This inequality is verified for a wide class of functions, including all the semi-algebraic ones~\cite{kurdyka_gradients_1998}, but it is not always obvious to know the exact expression of $\psi$ (see Section~\ref{subsec:kl} for a discussion). 

In the case where the KL inequality is verified with $\psi = cs^\alpha$ (the {\L}ojasiewicz case), we obtain by direct computation the following decay rate of the loss and convergence rate for the parameters:
 \begin{corollary}\label{col:loja_rate}
     If $\loss$ satisfies the \L ojasiewicz inequality, that is \ref{ass:l_kl} holds with $\psi(s) = cs^\alpha$ and $\alpha \in [0,1/2]$, then, $\exists t_0 > 0$ such that $\forall t \geq t_0, \gamma(t) > 0$, and thus the loss and the parameters converge with the following rates:
   
     \[
     \tabulinesep=1.2mm
    \begin{tabu}{rll}
    \lossy(\yvt) \leq &\left\{
    \begin{array}{l}
    \left(\frac{1 - 2\alpha}{\alpha^2c^2}\gamma(t)\right)^{-\frac{1}{1-2\alpha}} \qquad \\
    \exp\left(-\frac{4}{c^2}\gamma(t)\right) \qquad      
    \end{array}
    \right.
    &\begin{array}{l}
       \text{if } 0 < \alpha < \frac{1}{2}, \\
        \text{if } \alpha = \frac{1}{2} .
    \end{array}\\
    \norm{\thetavt - \thetav_{\infty}}\leq &
    \frac{2c}{\sigminjgz\sigminF}
    \left\{
    \begin{array}{l}
    \left(\frac{1 - 2\alpha}{\alpha^2c^2}\gamma(t)\right)^{-\frac{\alpha}{1-2\alpha}} \qquad \\
    \exp\left(-\frac{2}{c^2}\gamma(t)\right) \qquad      
    \end{array}
    \right.
    &\begin{array}{l}
       \text{if } 0 < \alpha < \frac{1}{2}, \\
        \text{if } \alpha = \frac{1}{2} .
    \end{array}
    \end{tabu}
\]

    
\end{corollary}

\begin{remark}
When $\psi(s) = cs^\alpha$ with $\alpha \in ]1/2,1]$, it is not difficult to see that one obtains a finite termination of gradient flow training, i.e.,  $\exists \hat{t} > 0$, that depends on $\alpha$, such that $\lossy(\yvt) = \norm{\thetavt - \thetav_{\infty}} = 0$ for all $t \geq \hat{t}$.
\end{remark}

\medskip

These results allow to see precise convergence rates of the loss for a wide variety of functions. First let us observe the particular case when $\alpha = 1/2$ which gives exponential convergence to the solution. The Mean Squared Error (MSE) loss corresponds precisely to this case. For $\alpha \in [0,1/2[$, we obtain a sublinear convergence rate as $O(t^{-\frac{1}{1 - 2\alpha}})$.
Observe also that the convergence rate of the parameters of the model is slower than that of the loss as the former is the latter powered by $\alpha < 1$.

\subsubsection{Early stopping strategy}
While the first result allows us to obtain convergence rates to a zero-loss solution, it does so by overfitting the noise inherent to the problem. A classical way to avoid this to happen is to use an early stopping strategy to ensure that our solution will lie in a ball around the desired solution. The bound on the time given in \eqref{eq:yrate} will verify that all the solutions found past that time will be no more than $2\norm{\veps}$ away from the noiseless solution. This bound is given by balancing the convergence rate offered by the KL properties of the loss, the loss of the model at initialization and the level of noise in the problem.


\subsubsection{Signal Recovery Guarantees}
Our third result provides a bound on the distance between the solution found at time $t$ and the true solution~$\xvc$. This bound is a sum of three terms representing three kinds of errors. The first term is an ``optimization error'', which represents how far $\xvt$ is from the solution found at the end of the optimization process. Of course, this decreases to 0 as $t$ goes to infinity. The second error is a ``modeling error'' which captures the expressivity of the optimized network, i.e. its ability to generate solutions close to $\xvc$. Finally, the third term is a ``noise error'' that depends on $\norm{\veps}$ which is inherent to the problem at hand.   

Obviously, the operator $\fopnl$ also plays a key role in this bound where its influence is reflected by three quantities of interest: $\sigminF$, $\LF$ and $\muf$. First, $\LF$ is the Lipschitz constant of the Jacobian of $\fopnl$ on $\Sigma'$. Moreover, we always have $\sigminF > 0$ and the dependence of the bound on $\sigminF$ (or the ratio $\LF/\sigminF$) reflects the fact that this bound degrades as the Jacobian of $\fopnl$ over $\Sigma_\Theta$ becomes badly-conditioned. Second, $\muf$ corresponds to a restricted injectivity condition, which is a classical and natural assumption if one hopes for recovering $\xvc$ (to a good controlled error). 
In particular, in the case where $\fopnl$ is a linear operator $\fop\in\R^{m\times n}$, $\muf$ becomes the minimal conic singular value $\lmin(\fop;T_{\Sigma'}(\xvcsigma))$ and $\LF$ is replaced by $\norm{\fop}$. \eqref{ass:F_inj} then amounts to assuming that
\begin{equation}\label{ass:lin_inj}
\ker(\fop) \cap T_{\Sigma'}(\xvcsigma) = \ens{0} .
\end{equation}
Assuming the rows of $\fop$ are linearly independent, one easily checks that \eqref{ass:lin_inj} imposes that $m \geq \dim(T_{\Sigma'}(\xvcsigma))$. We will give a precise sample complexity bound for the case of compressed sensing in Example~\ref{ex:CS}. It is worth mentioning that condition \eqref{ass:lin_inj} (and \eqref{ass:F_inj} in some sense) is not uniform as it only requires a control at $\xvc$ and not over the whole set $\Sigma'$.


Observe that the restricted injectivity condition \eqref{ass:F_inj} depends on $\Sigma'$ which itself depends on $R'$, that is, the radius of the ball around $\thetav_0$ containing the whole trajectory $\theta(t)$ during the network training (see the proof of Lemma~\ref{lemma:link_params_singvals}). On the other hand, $R'$ depends on the loss at initialization, which means that the higher the initial error of the network, the larger the set of parameters it might reach during optimization, and thus the larger the set $\Sigma'$. This discussion clearly reveals an expected phenomenon: there is a trade-off between the restricted injectivity condition on $\fopnl$ and the expressivity of the network. If the model is highly expressive then $\dist(\xvc,\Sigma')$ will be smaller. But this is likely to come at the cost of making $\muf$ decrease, as restricted injectivity can be required to hold on a larger subset (cone).


This discussion relates with the work on the instability phenomenon observed in learned reconstruction methods as discussed in \cite{antun2020instabilities,gottschling2020troublesome}. For instance, when $\fopnl$ is a linear operator $\fop$, the fundamental problem that creates these instabilities and/or hallucinations in the reconstruction is due to the fact that the kernel of $\fop$ is non-trivial.
Thus a method that can correctly learn to reconstruct signals whose difference lies in or close to the kernel of $\fop$ will necessarily be unstable or hallucinate. In our setting, this is manifested through the  restricted injectivity condition, that imposes that the smallest conic singular value is bounded away from $0$, i.e. $\muf=\lmin(\fop;T_{\Sigma'}(\xvcsigma)) > 0$. This is a natural (and minimal) condition in the context of inverse problems to have stable reconstruction guarantees. Note that our condition is non-uniform as it is only required to hold at $\xvcsigma$ and not at all points of $\Sigma'$. 

In \ref{ass:F_jac_inj}, we generalize the restricted injectivity condition \eqref{ass:lin_inj} beyond the linear case provided that $\Jf$ is Lipschitz continuous. This covers many practical cases, for instance that of phase retrieval. Observe  that whereas assumption~\ref{ass:F_inj} requires a uniform control of injectivity of $\fopnl$ on the whole signal class $\Sigma'$, \ref{ass:F_jac_inj} is less demanding and only requires injectivity of the Jacobian of $\fopnl$ at $\xvcsigma$ on the tangent space of $\Sigma'$ at $\xvcsigma$. However the price is that the recovery bound in Theorem~\ref{thm:alternative_rec_bound} is only valid for high signal-to-noise regime and $\dist(\xvc,\Sigma')$ is small enough. Moreover, the convergence rate in noise becomes $O(\sqrt{\norm{\veps}})$ which is worse than $O({\norm{\veps}})$ of Theorem~\ref{thm:main}.

\medskip

\begin{example}[Compressed sensing with sub-Gaussian measurements]\label{ex:CS} 
Controlling the minimum conic singular value is not easy in general. Amongst the cases where results are available, we will look at the compressed sensing framework with linear random measurements. In this setting, the forward operator $\fop \in \R^{m\times n}$ is a random sensing matrix. Exploiting the randomness of $\fop$, a natural question is then how many measurements are sufficient to ensure that $\lmin(\fop;T_{\Sigma'}(\xvcsigma)) > 0$ with high probability. In the case of Gaussian and sub-Gaussian measurements, we can exploit the non-uniform results of \cite{chandrasekaran2012convex,tropp2015convex} to derive sample complexity bounds, i.e. lower bounds on $m$, for this to hold. Building upon~\cite[Theorem 6.3]{tropp2015convex}, we have the following result:
\begin{proposition}\label{prop:compressed_sensing}
    Assume that each row $\fop^i$ is an independent sub-Gaussian vector, that is
\begin{enumerate}[label=(\roman*)]
    \item $\E[\fop^i] = 0$,
    \item $\alpha \leq \E[\abs{\dotprod{\fop^i}{\wv}}]$ for each $\wv \in \sph^{n-1}$ with $\alpha > 0$,
    \item $\prob{\abs{\dotprod{\fop^i}{\wv}} \geq \tau} \leq 2e^{-\tau^2/(2\sigma^2)}$ for each $\wv \in \sph^{n-1}$, with $\sigma > 0$.
\end{enumerate}
Let $C$ and $C'$ be positive constants and $w(K)$ the Gaussian width of the cone $K$ defined as:
\begin{align*}
    w(K) = \E_{\zv \sim \mathcal{N}(0,\mathbf{I})}\left[\sup_{\wv\in K\cap \sph^{d-1}}\dotprod{\zv}{\wv}\right].
\end{align*}
If
\begin{align*}
    m \geq C'\left(\frac{\sigma}{\alpha}\right)^6 w(T_{\Sigma'}(\xvcsigma))^2 + 2C^{-2} \frac{\sigma^2}{\alpha^4}\tau^2,
\end{align*}
then $\lmin(\fop,T_{\Sigma'}(\xvcsigma)) > 0$ with probability at least $1 - \exp(-C\tau^2)$.
\end{proposition}

{~}

The Gaussian width is an important tool in high-dimensional convex geometry and can be interpreted as a measure of the ``dimension'' of a cone. Except in some specific settings (such as when $K$ is a descent cone of a convex function and other special cases), it is notoriously difficult to compute this quantity; see the discussion in \cite{chandrasekaran2012convex}. Another ``generic'' tool for computing Gaussian widths is based on Dudley’s inequality which
bounds the width of a set in terms of the covering number of the set at all scales. Estimating the covering number is not easy either in general. This shows the difficulty of computing  $w(T_{\Sigma'}(\xvcsigma))$ which we leave to a future work.

Analyzing recovery guarantees in the compressed sensing framework using unsupervised neural networks such as DIP was proposed in \cite{joshi2021plugin,jagatap2019algorithmic}. In~\cite{joshi2021plugin}, the authors restricted their analysis to the case of networks without non-linear activations nor training/optimization. The authors of~\cite{jagatap2019algorithmic} studied the case of the DIP method but their optimization algorithms is prohibitively intensive necessitating at each iteration retraining the DIP network. Another distinctive difference with our work is that these existing results are uniform relying on RIP-type arguments and their specialization for Gaussian measurements.  

    
\end{example}



\subsubsection{Existence and Uniqueness of a Global Strong Solution}\label{subsec:wellglobal}
We have already stated in Section~\ref{subsec:welllocal} that~\eqref{eq:gradflow} admits a unique maximal solution. Assumption~\eqref{eq:bndR} allows us to further specify this solution as strong and global. Indeed, \eqref{eq:thetarate} ensures that the trajectory $\thetav(t)$ is uniformly bounded. Let us start by recalling the notion of a strong solution.
\begin{definition}\label{def:strongsol}
Denote $\thetav: t \in [0,+\infty[ \mapsto \thetav(t) \in \R^p$. The function $\thetav(\cdot)$ is a strong global solution of \eqref{eq:gradflow} if it satisfies the following properties:
\begin{enumerate}[label=$\bullet$]
\item $\thetav$ is in $\cC^1([0,+\infty[;\R^p)$;
\item for almost all $t \in [0,+\infty[$, \eqref{eq:gradflow} holds with $\thetav(0) = \thetavz$.
\end{enumerate}
\end{definition}

\begin{proposition}\label{prop:wellglobal}
Assume that \ref{ass:l_smooth}-\ref{ass:F_diff} and \eqref{eq:bndR} are satisfied. Then, for any initial condition $\thetavz$, the evolution system \eqref{eq:gradflow} has a unique strong global solution.
\end{proposition}

\begin{proof}
Proposition~\ref{prop:welllocal} ensures the existence and uniqueness of a maximal solution. Following the discussion after the proof of Proposition~\ref{prop:welllocal}, if $\thetav(t)$ is bounded, then we are done. This is precisely what is ensured by Theorem~\ref{thm:main} under our conditions.
\end{proof}

\subsection{Proofs}\label{sec:proof}

We start with the following lemmas that will be instrumental in the proof of Theorem~\ref{thm:main}.

\begin{lemma}\label{lem:descentflow}
Assume that \ref{ass:l_smooth}, \ref{ass:min_l_zero}, \ref{ass:phi_diff} and \ref{ass:F_diff} hold. Let $\thetav(\cdot)$ be a solution trajectory of \eqref{eq:gradflow}. Then,
\begin{enumerate}[label=(\roman*)]
\item $\lossy(\yv(\cdot)))$ is nonincreasing, and thus converges. \label{lem:descentflowclaim1}
\item If $\thetav(\cdot)$ is bounded, $\lossy(\yv(\cdot)))$ is constant on $\cluster{\thetav(\cdot)}$. \label{lem:descentflowclaim2}
\end{enumerate}
\end{lemma}

\begin{proof}
Let $\Vt = \lossyty$. 
\begin{enumerate}[label=(\roman*)]
\item Differentiating $V(\cdot)$, we have for $t > 0$:
\begin{align}
    \Vdt &= \dotprod{\dot{\yv}(t)}{\lossgrad} \nonumber\\
    & = \dotprod{\Jft\Jgt\dot{\pmb{\theta}}(t)}{\lossgrad}\nonumber\\
    &= - \dotprod{\Jft\Jgt\Jgt\tp\Jft\tp\lossgrad}{\lossgrad}\nonumber\\
    &= - \norm{\Jgt\tp\Jft\tp\lossgrad}^2 = - \norm{\dot{\thetav}(t)}^2 , \label{eq:odeVt}
\end{align}
and thus $V(\cdot)$ is decreasing. Since it is bounded from below (by $0$ by assumption), it converges to say $\loss_\infty$ (0 in our case). 
\item Since $\thetav(\cdot)$ is bounded, $\cluster{\thetav(\cdot)}$ is non-empty. Let $\thetav_{\infty} \in \cluster{\thetav(\cdot)}$. Then $\exists t_k \to +\infty$ such that $\thetav(t_k) \to \thetav_{\infty}$ as $k \to +\infty$. Combining claim \ref{lem:descentflowclaim1} with continuity of $\loss$, $\fopnl$ and $\gv(\cdot,\uv)$, we have
\[
\loss_\infty = \lim_{k \to +\infty} \lossy(\fopnl(\gv(\uv,\thetav(t_k)))) = \lossy(\fopnl(\gv(\uv,\thetav_{\infty}))) .
\]
Since this is true for any cluster point, the claim is proved.
\end{enumerate}
\end{proof}



\begin{lemma}
\label{lemma:theta_summable}
Assume that \ref{ass:l_smooth} to \ref{ass:F_diff} hold. Let $\thetav(\cdot)$ be a solution trajectory of \eqref{eq:gradflow}. If for all $t \geq 0$, $\sigminjgt \geq \frac{\sigminjgz}{2} > 0$, then $\lVert{\dot{\thetav}(\cdot)}\rVert \in L^1([0,+\infty[)$. In turn, $\lim_{t \to +\infty}\thetav(t)$ exists.
\end{lemma}
\begin{proof}
From Lemma~\ref{lem:descentflow}\ref{lem:descentflowclaim1}, we have for $t \geq 0$:
\[
\yv(t) \in [0 \leq \lossy(\cdot) \leq \lossy(\yvz)] .
\]
We may assume without loss of generality that $\yv(t) \in [0 < \lossy(\cdot) \leq \lossy(\yvz)]$ since otherwise $\lossy(\yv(\cdot))$ is eventually zero which implies, by Lemma~\ref{lem:descentflow}, that $\dot{\thetav}$ is eventually zero, in which case there is nothing to prove.

We are now in position to use the KL property on $\yv(\cdot)$. We have for $t > 0$:
\begin{align}
    \deriv{\psi(\lossyty)}{t} &= \psi'(\lossyty)\deriv{\lossyty}{t} \nonumber\\
    &= -\psi'(\lossyty)\norm{\Jgt\tp\Jft\tp\lossgrad}^2 \nonumber\\
    &\leq - \frac{\norm{\Jgt\tp\Jft\tp\lossgrad}^2}{\norm{\lossgrad}} \nonumber\\
    &\leq -\sigminjgt\sigminF\norm{\Jgt\tp\Jft\tp\lossgrad} \nonumber\\
    &\leq - \frac{\sigminjgz\sigminF}{2}\norm{\dot{\thetav}(t)}. \label{eq:dottheta_bnd}
\end{align}
where we used \ref{ass:nablaL_F}, that $\sigminjgt \geq \frac{\sigminjgz}{2} > 0$ and~\eqref{eq:gradflow}.
Integrating, we get
\begin{align}\label{eq:dottheta_summable}
    \int_{0}^t \norm{\dot{\thetav}(s)} \ds \leq  \frac{2}{\sigminjgz\sigminF}\pa{\psi(\lossy(\yv(0))) - \psi(\lossy(\yv(t)))} .
\end{align}
Since $\lossyty$ converges thanks to Lemma~\ref{lem:descentflow}\ref{lem:descentflowclaim1} and $\psi$ is continuous and increasing, the right hand side in \eqref{eq:dottheta_summable} has a limit. Thus passing to the limit as $t \to +\infty$, we get that $\dot{\thetav} \in L^1([0,+\infty[)$. This in turn implies that $\lim_{t \to +\infty}\thetav(t)$ exists, say $\thetav_{\infty}$, by applying Cauchy's criterion to
\[
\thetav(t) = \thetavz + \int_0^t \dot{\thetav}(s) \ds .
\]
\end{proof}

\begin{lemma}\label{lemma:link_params_singvals}
Assume that \ref{ass:l_smooth} to \ref{ass:F_diff} hold. Recall $R$ and $R'$ from \eqref{eq:RandR'}. Let $\thetav(\cdot)$ be a solution trajectory of \eqref{eq:gradflow}.
    \begin{enumerate}[label=(\roman*)]
        \item \label{claim:singvals_bounded_if_params_bounded} If $\thetav \in \Ball(\thetavz,R)$ then
        \begin{align*}
            \sigmin(\jtheta) \geq \sigminjgz/2.
        \end{align*}
        
        \item  \label{claim:params_bounded_if_singvals_bounded} If for all $s \in [0,t] $, $\sigminjgs \geq \frac{\sigminjgz}{2}$ then 
            \begin{align*}
                \thetavt \in \Ball(\thetavz,R') .
            \end{align*}
        
        \item \label{claim:sigval_bounded_everywhere}
        If $R'<R$, then for all $t \geq 0$, $\sigminjgt \geq \sigminjgz/2$.
\end{enumerate}
\end{lemma}
\begin{proof}
\begin{enumerate}[label=(\roman*)]
\item Since $\thetav \in \Ball(\thetavz,R)$, we have
    \begin{align*}
        \norm{\jtheta - \jthetaz} \leq \Lip_{\Ball(\thetavz,R)}(\Jg)\norm{\thetav - \thetavz} \leq \Lip_{\Ball(\thetavz,R)}(\Jg)R \leq \frac{\sigminjgz}{2}.
    \end{align*}
    By using that $\sigmin(\cdot)$ is 1-Lipschitz, we obtain
    \begin{align*}
        \sigmin(\jtheta) \geq \sigminjgz - \norm{\jtheta - \jthetaz} \geq \frac{\sigminjgz}{2} .
    \end{align*}

\item 
We have for $t > 0$
\[
\norm{\thetav(t) - \thetavz} = \norm{\int_{0}^{t} \dot{\thetav}(s) \ds} \leq \int_{0}^{t} \norm{\dot{\thetav}(s)} \ds .
\]
Combining this with \eqref{eq:dottheta_summable} yields
\begin{align*}
    \norm{\thetav(t) - \thetavz} \leq \int_{0}^{t} \norm{\dot{\thetav}(s)} \ds
    \leq \frac{2}{\sigminjgz\sigminF} \psi(\lossy(\yv(0))) ,
\end{align*}
where we argue that $\lossy(\yv(t))$ is positive and bounded and $\psi$ is positive and increasing.

\item Actually, we prove the stronger statement that $\thetavt \in \Ball(\thetavz,R')$ for all $t \geq 0$, whence our claim will follow thanks to \ref{claim:singvals_bounded_if_params_bounded}. Let us assume for contradiction that $R'<R$ and $\exists~ t < +\infty$ such that $\thetavt \notin \Ball(\thetavz,R')$. By \ref{claim:params_bounded_if_singvals_bounded}, this means that $\exists~ s \leq t$ such that $\sigminjgs < \sigminjgz/2$. In turn, \ref{claim:singvals_bounded_if_params_bounded} implies that $\thetav(s) \notin \Ball(\thetavz,R)$. Let us define
    \begin{align*}
        t_0 = \inf\{\tau \geq 0:\thetav(\tau) \notin \Ball(\thetavz,R)\},
    \end{align*}
    which is well-defined as it is at most $s$. Thus, for any small $\epsilon > 0$ and for all $t' \leq t_0 - \epsilon$, $\thetav(t') \in \Ball(\thetavz,R)$ which, in view of \ref{claim:singvals_bounded_if_params_bounded} entails that $\sigmin(\jtheta(t')) \geq \sigminjgz/2$. In turn, we get from \ref{claim:params_bounded_if_singvals_bounded} that $\thetav(t_0-\epsilon) \in \Ball(\thetavz,R')$. Since $\epsilon$ is arbitrary and $\thetav$ is continuous, we pass to the limit as $\epsilon \to 0$ to deduce that $\thetav(t_0) \in \Ball(\thetavz,R') \subsetneq \Ball(\thetavz,R)$ hence contradicting the definition of $t_0$.
\end{enumerate}
\end{proof}

\begin{proof}[Proof of Theorem~\ref{thm:main}]
\begin{enumerate}[label=(\roman*)]
\item We here use a standard Lyapunov analysis with several energy functions. Let us reuse $\Vt$. Embarking from \eqref{eq:odeVt}, we have for $t > 0$
\begin{align*}
    \Vdt 
    &= - \norm{\Jgt\tp\Jft\tp\lossgrad}^2\\
    &\leq -\sigminjgt^2\sigminF^2\norm{\lossgrad}^2 ,
\end{align*}
where we used \ref{ass:nablaL_F}. In view of Lemma~\ref{lemma:link_params_singvals}\ref{claim:sigval_bounded_everywhere}, we have $\sigminjgt \geq \sigminjgz/2 > 0$ for all $t \geq 0$ if the initialization error verifies \eqref{eq:bndR}. Using once again \ref{ass:l_kl}, we get
\begin{align*}
    \Vdt &\leq -\frac{\sigminjgz^2\sigminF^2}{4}\norm{\lossgrad}^2\\
    &\leq -\frac{\sigminjgz^2\sigminF^2}{4\psi'(\lossyty)^2}.
\end{align*}
Let $\Psi$ be a primitive of $-\psi'^2$. Then, the last inequality gives 
\begin{align*}
    \dot{\Psi}(\Vt) &= \Psi'(\Vt)\Vdt\\
    &\geq \frac{\sigminF^2\sigminjgz^2}{4}.
\end{align*}
By integration on $s\in[0,t]$ alongside the fact that $\Psi$ and $\Psi\inv$ are (strictly) decreasing functions, we get
\begin{align*}
    \Psi(\Vt) - \Psi(V(0)) &\geq \frac{\sigminF^2\sigminjgz^2}{4}t\\
    V(t) &\leq \Psi\inv\left(\frac{\sigminF^2\sigminjgz^2}{4}t + \Psi(V(0))\right),
\end{align*}
which gives \eqref{eq:lossrate}.

By Lemma~\ref{lemma:theta_summable}, $\thetav(t)$ converges to some $\thetav_{\infty}$. Continuity of $\lossy(\cdot)$, $\fopnl$ and $\gv(\uv,\cdot)$ implies that
\[
0  = \lim_{t \to +\infty}\lossy(\yv(t)) = \lim_{t \to +\infty} \lossy(\fopnl(\gv(\uv,\thetav(t)))) = \lossy(\fopnl(\gv(\uv,\thetav_{\infty}))) ,
\]
and thus $\thetav_{\infty} \in \Argmin(\lossy(\fopnl(\gv(\uv,\cdot))))$. To get the rate, we argue as in the proof of Lemma~\ref{lemma:link_params_singvals} \ref{claim:params_bounded_if_singvals_bounded}, replacing $\thetavz$ by $\thetav_{\infty}$, to obtain
\begin{align*}
    \norm{\thetav(t) - \thetav_{\infty}} &\leq \int_{t}^{+\infty} \norm{\dot{\thetav}(s)}\ds .
\end{align*}
We then get by integrating \eqref{eq:dottheta_bnd} that
\begin{align*}
    \norm{\thetav(t) - \thetav_{\infty}} 
    &\leq - \frac{2}{\sigminjgz\sigminF} \int_{t}^{+\infty} \deriv{\psi(\lossy(\yv(s)))}{s}\ds\\
    &\leq \frac{2}{\sigminjgz\sigminF} \psi(\lossy(\yv(t))).
\end{align*}
Thanks to \eqref{eq:lossrate}, and using that $\psi$ is increasing, we arrive at \eqref{eq:thetarate}.

\item By Lemma~\ref{lemma:theta_summable} and continuity of $\fopnl$ and $\gv(\uv,\cdot)$, we can infer that $\yv(\cdot)$ also converges to $\yv_{\infty}=\fopnl(\gv(\uv,\thetav_{\infty}))$, where $\thetav_{\infty}=\lim_{t \to +\infty} \thetav(t)$. Thus using also continuity of $\lossy(\cdot)$, we have
\[
0  = \lim_{t \to +\infty} \lossyty = \lossy(\yv_{\infty}) ,
\]
and thus $\yv_{\infty} \in \Argmin(\lossy)$. Since the latter is the singleton $\ens{\yv}$ by assumption, we conclude.

\modifs{In order to obtain the early stopping bound, we use~\cite[Corollary~6(i)]{bolte2017error} that links the (global) KL property of $\lossy(\cdot)$ with an error bound. In our case, this entails that for all $t \geq 0$,}
\begin{align}\label{eq:klerrbnd}
\dist(\yvt,\Argmin(\lossy)) = \norm{\yvt - \yv} \leq \psi(\lossyty).
\end{align}
It then follows that
\begin{align*}
    \norm{\yvt - \yvc} &\leq \norm{\yvt - \yv} + \norm{\yv - \yvc}\\
    &\leq \psi(\lossyty) + \norm{\veps}\\
    &\leq \psi\left(\Psi\inv\left(\frac{\sigminF^2\sigminjgt^2}{4}t + \Psi(V(0))\right)\right) + \norm{\veps}.
\end{align*}
Using that $\psi$ is increasing and $\Psi$ is decreasing, the first is bounded by $\norm{\veps}$ for all $t\geq \frac{4\Psi(\psi\inv(\norm{\veps}))}{\sigminF^2\sigminjgz^2}-\Psi(V(0))$.


\item 

We recall that $\thetav(t) \in \Ball_{R'}(\thetavz)$ by Lemma~\ref{lemma:link_params_singvals}, which in turn entails that $\xvt \in \Sigma'$ for all $t \geq 0$. We then have the following chain of inequalities
\begin{align*}
    \norm{\xvt - \xvc} 
    &\leq \norm{\xvt - \xvcsigma} + \dist(\xvc,\Sigma') \\
\text{\small\ref{ass:F_inj}}    
	&\leq \muf^{-1}\norm{\yvt - \fopnl(\xvcsigma)} + \dist(\xvc,\Sigma') \\
    &\leq \muf^{-1}\pa{\norm{\yvt - \yv} + \norm{\yv - \fopnl(\xvc)} + \norm{\fopnl(\xvc) - \fopnl(\xvcsigma)}} +\dist(\xvc,\Sigma') \\
{\small\eqref{eq:forward}, \eqref{eq:lossrate}, \eqref{eq:klerrbnd}}    
	&\leq \frac{2\psi\pa{\Psi\inv\pa{\gamma(t)}}}{\muf\sigminjgz\sigminF} + \muf^{-1}\norm{\veps} + \norm{\fopnl(\xvc) - \fopnl(\xvcsigma)} + \dist(\xvc,\Sigma') .
\end{align*}
By assumption \ref{ass:F_diff} and the mean value theorem, we have
\[
\norm{\fopnl(\xvc) - \fopnl(\xvcsigma)} \leq \max_{\zv \in [\xvc,\xvcsigma]} \norm{\Jf(\zv)}\dist(\xvc,\Sigma') .
\]
Since $0 \in \Sigma'$, by Jensen's inequality, we have for all $\zv \in [\xvc,\xvcsigma]$ and $\rho \in [0,1]$:
\[
\norm{\zv} \leq \norm{\xvc} + \rho\dist(\xvc,\Sigma') \leq 2\norm{\xvc} + \norm{\xvz},
\]
meaning that $[\xvc,\xvcsigma] \subset \Ball(0,2\norm{\xvc})$. Thus
\begin{align}\label{eq:diff_Fx_Fxsigma}
\norm{\fopnl(\xvc) - \fopnl(\xvcsigma)} \leq \max_{\zv \in \Ball(0,2\norm{\xvc})} \norm{\Jf(\zv)} \dist(\xvc,\Sigma') .
\end{align}

\end{enumerate}
\end{proof}




%

%% file: tex/sec_DIP.tex
\section{Case of The Two-Layer DIP Network}\label{sec:dip}

This section is devoted to studying under which conditions on the neural network architecture the key condition in \eqref{eq:bndR} is fulfilled. Towards this goal, we consider the case of a two-layer DIP network. Therein, $\uv$ is randomly set and kept fixed during the training, and the network is trained to transform this input into a signal that matches the observation $\yv$. In particular, we will provide bounds on the level of overparametrization ensuring that \eqref{eq:bndR} holds, which in turn will provide the subsequent recovery guarantees in Theorem~\ref{thm:main}. 

\subsection{The Two-Layer Neural Network}
We take $L=2$ in Definition~\ref{def:nn} and thus consider the network defined in~\eqref{eq:dipntk}:
\begin{align*}
    \gdip = \frac{1}{\sqrt{k}}\Vv\phi(\Wv\uv)
\end{align*}
with  $\Vv \in \R^{n \times k}$ and $\Wv \in \R^{k \times d}$, and $\phi$ an element-wise nonlinear activation function. Observe that it is immediate to account for the bias vector in the hidden layer by considering the bias as a column of the weight matrices $\Wv$, augmenting $\uv$ by 1 and then normalizing to unit norm. The normalization is required to comply with~\ref{ass:u_sphere} hereafter. The role of the scaling by $\sqrt{k}$ will become apparent shortly, but it will be instrumental to concentrate the kernel stemming from the jacobian of the network. 


In the sequel, we set $\Cphi = \sqrt{\Expect{X\sim\stddistrib}{\phi(X)^2}}$ and $\Cphid = \sqrt{\Expect{X\sim\stddistrib}{\phi'(X)^2}}$. We will assume without loss of generality that $\fopnl(0)=0$. This is a very mild assumption that is natural in the context of inverse problems, but can be easily removed if needed. We will also need the following assumptions:
\begin{mdframed}[frametitle={Assumptions on the network input and intialization}]
\begin{assumption}\label{ass:u_sphere}
$\uv$ is a uniform vector on $\sph^{d-1}$;
\end{assumption}
\begin{assumption}\label{ass:w_init}
$\Wv(0)$ has iid entries from $\stddistrib$ and $\Cphi, \Cphid < +\infty$; 
\end{assumption}
\begin{assumption}\label{ass:v_init}
$\Vv(0)$ is independent from $\Wv(0)$ and $\uv$ and has iid columns with identity covariance and $D$-bounded centered entries.
\end{assumption}
\end{mdframed}

\subsection{Recovery Guarantees in the Overparametrized Regime}
%
Our main result gives a bound on the level of overparameterization which is sufficient for \eqref{eq:bndR} to hold. 

\begin{theorem}\label{th:dip_two_layers_converge}
Suppose that assumptions \ref{ass:l_smooth}, \ref{ass:min_l_zero}, \ref{ass:phi_diff} and \ref{ass:F_diff} hold. Let $C$, $C'$ two positive constants that depend only on the activation function and $D$. Let:
\[
\LFz = \max_{\xv \in \Ball\pa{0,C\sqrt{n\log(d)}}} \norm{\Jf(\xv)}
\]
and
\[
\LLz = \max_{\vv \in \Ball\pa{0,C\LFz\sqrt{n\log(d)}+\sqrt{m}\pa{\norminf{\fopnl(\xvc)}+\norminf{\veps}}}} \frac{\norm{\nabla_\vv \lossy(\vv)}}{\norm{\vv-\yv}}.
\]
Consider the one-hidden layer network \eqref{eq:dipntk} where both layers are trained with the initialization satisfying \ref{ass:u_sphere} to \ref{ass:v_init} and the architecture parameters obeying
\begin{align*}
k \geq C' \sigminF^{-4} n \psi\pa{\frac{\LLz}{2}\pa{C\LFz \sqrt{n\log(d)} + \sqrt{m}\pa{\norminf{\fopnl(\xvc)} + \norminf{\veps}}}^2}^4 .
\end{align*}
Then \eqref{eq:bndR} holds with probability at least $1 - 2n^{-1} - d^{-1}$.
\end{theorem}

Before proving Theorem~\ref{th:dip_two_layers_converge}, a few remarks are in order.

\begin{remark}[Randomness of $\Sigma'$]
It is worth observing that since the initialization is random, so is the set of signals $\Sigma'=\Sigma_{\Ball_{R'+\norm{\thetavz}}(0)}$ by definition, where $\thetavz=(\Vv(0),\Wv(0))$. This set is contained in a larger  deterministic set with high probability. Indeed, Gaussian concentration gives us, for any $\delta > 0$,
\[
\normf{\Wv(0)} \leq (1+\delta)\sqrt{kd}
\]
with probability larger than $1-e^{-\delta^2 kd/2}$. Moreover, since by \ref{ass:v_init} $\Vv(0)$ has independent columns with bounded entries and $\Expect{}{\norm{\Vv_i(0)}^2}=n$, we can apply Hoeffding's inequality to $\normf{\Vv(0)}^2=\sum_{i=1}^k\norm{\Vv_i(0)}^2$ to infer that
\[
\normf{\Vv(0)} \leq (1+\delta)\sqrt{kn}
\]
with probability at least $1-e^{-\delta^2 kd/(2D^2)}$. Collecting the above, we have
\[
\norm{\thetavz} \leq (1+\delta)\sqrt{k}\pa{\sqrt{n}+\sqrt{d}} ,
\]
with probability at least $1-e^{-\delta^2 kd/2}-e^{-\delta^2 kd/(2D^2)}$. In view of the bound on $R'$ (see \eqref{eq:R'Rbnd}), this yields that with probability at least $1-e^{-\delta^2 kd/2}-e^{-\delta^2 kd/(2D^2)}-2n^{-1}-d^{-1}$, $\Sigma' \subset \Sigma_{\Ball_{\rho}(0)}$, where
\begin{multline*}
\rho = \frac{4}{\sigminF\sqrt{\Cphi^2 + \Cphid^2}}\psi\pa{\frac{\LLz}{2}\pa{C\LFz \sqrt{n\log(d)} + \sqrt{m}\pa{\norminf{\fopnl(\xvc)} + \norminf{\veps}}}^2} \\
+ (1+\delta)\sqrt{k}\pa{\sqrt{n}+\sqrt{d}} .
\end{multline*}
This confirms the expected behaviour that expressivity of $\Sigma'$ is higher as the overparametrization increases.
\end{remark}

\begin{remark}[Distribution of $\uv$]
The generator $\gv(\cdot,\thetav)$ synthesize data by transforming the input (latent) random variable $\uv$. As such, it generates signals $\xv \in \Sigma'$ who are in the support of the measure $\gv(\cdot,\thetav) \# \mu_{\uv}$, where $\mu_{\uv}$ is the distribution of $\uv$, and $\#$ is the push-forward operator. Expressivity of these generative models, coined also push-forward models, in particular GANs, have been recently studied either empirically or theoretically \cite{Khayatkhoei,Gurumurthy,Tanielian,Salmona,Issenhuth}. In particular, this literature highlights the known fact that, since $\gv(\cdot,\thetav)$ is continuous by construction, the support of $\gv(\cdot,\thetav) \# \mu_{\uv}$ is connected if that of $\mu_{\uv}$ is connected (as in our case). On the other hand, a common assumption in the imaging literature, validated empirically by \cite{Pope}, is that distributions of natural images are supported on low dimensional manifolds. It is also conjectured that the distribution of natural images may in fact lie on a union of disjoint manifolds rather than one globally connected manifold; the union of subspaces or manifolds model is indeed a common assumption in signal/image processing. In the latter case, a generator $\gv(\cdot,\thetav)$ that will attempt to cover the different modes (manifolds) of the target distribution from one unimodal latent variable $\uv$ will generate samples out of the real data manifold. There are two main ways to avoid this: either making the support of $\mu_{\uv}$ disconnected (e.g. using a mixture of distributions \cite{Gurumurthy,Hagemann}), or making $\gv(\cdot,\thetav)$ discontinuous \cite{Khayatkhoei}. The former strategy appears natural in our context and it will be interesting to investigate this generalization in a future work.
\end{remark}

\begin{remark}[Restricted injectivity]
As argued above, if $\Sigma'$ belongs to a target manifold $\cM$, then the restricted injectivity condition \eqref{ass:lin_inj} tells us that $\fop$ has to be invertible on the tangent cone of the target manifold $\cM$ at the closest point of $\xvc$ in $\cM$. 
\end{remark}

\begin{remark}[Dependence on $\LLz$ and $\LFz$]
The overparametrization bound on $k$ depends on $\LLz$ and $\LFz$ which in turn may depend on $(n,m,d)$. Their estimate is therefore important. For instance, if $\fopnl$ is globally Lipschitz, as is the case when it is linear, then $\LFz$ is independent of $(n,m,d)$. As far as $\LLz$ is concerned, it is of course independent of $(n,m,d)$ if the loss gradient is globally Lipschitz continuous. Another situation of interest is when $\nabla_{\vv}\lossy(\vv)$ verifies
\[
\norm{\nabla_\vv\lossy(\vv) - \nabla_\zv\lossy(\zv)} \leq \varphi\pa{\norm{\vv-\zv}}, \quad \forall \vv,\zv \in \R^m ,
\]
where $\varphi: \R_+ \to \R_+$ is increasing and vanishes at $0$.
This is clearly weaker than global Lipschitz continuity and covers it as a special case. It also encompasses many important situations such as e.g. losses with H\"olderian gradients. It then easily follows, see e.g. \cite[Theorem~18.13]{BauschkeBook}, that for all $\vv \in \R^m$:
\[
\lossy(\vv) \leq \Phi\pa{\norm{\vv-\yv}} \qwhereq \Phi(s) = \int_0^1 \frac{\varphi(st)}{t} \dt .
\]
In this situation, and if $\fopnl$ is also globally $\LF$-Lipschitz, following our line of proof, the overparametrization bound of Theorem~\ref{th:dip_two_layers_converge} reads
\begin{align*}
k \geq C' \sigminF^{-4} n \psi\pa{\Phi\pa{C\LF \sqrt{n\log(d)} + \sqrt{m}\pa{\norminf{\fopnl(\xvc)} + \norminf{\veps}}}}^4 .
\end{align*}
\end{remark}


\begin{remark}[Dependence on the loss function]
    If we now take interest in the scaling of the overparametrization bound on $k$ with respect to $(n,m,d)$ in the general case we obtain that $k\gtrsim \sigminF^{-4} n \psi(\LLz(\LFz^2 n + m))^4$. Aside from the possible dependence of $\LLz$ and $\LFz$ on the parameters $(n,m,d)$ discussed before, we observe that this bound is highly dependent on the desingularizing function $\psi$ given by the loss function. In the \L ojasiewicz case where $\psi = cs^\alpha$ with $\alpha \in [0,1]$, one can choose to use a sufficiently small $\alpha$ to reduce the scaling on the parameters but then one would slow the convergence rate as described in Corollary~\ref{col:loja_rate} which implies a tradeoff between the convergence rate and the number of parameters to ensure this convergence.
    
    In the special case where $\alpha = \frac{1}{2}$ which corresponds to the MSE loss, and where  $\LFz$ is of constant order and independent of $(n,m,d)$, then the overperametrization of $k$ necessary for ensuring convergence to a zero-loss is $k\gtrsim n^3m^2$. Another interesting case is when $\fopnl$ is linear. In that setting, the overparametrization bound becomes $k\gtrsim \sigminF^{-4} n \psi(\LLz(\norm{\fopnl}^2 n + m))^4$. By choosing the MSE loss, and thus controlling $\psi$ to be a square root operator, then we obtain that we need $k\gtrsim \kappa(\fopnl)^4 n^3m^2$. The bound is thus more demanding as $\fopnl$ becomes more and more ill-conditioned. The latter dependency can be interpreted as follows: the more ill-conditioned the original problem is, the larger the network needs to be. 
\end{remark}
\begin{remark}[Scaling when $\Vv$ is fixed]\label{rq:V_fixed}
    When the linear layer $\Vv$ is fixed and only $\Wv$ is trained, the overparametrization bound to guarantee convergence can be improved (see Appendix~\ref{appendix:V_fixed} and the results in~\cite{buskulic2023convergence}). In this case, one needs $k\gtrsim \sigminF^{-2} n \psi(\LLz(\LFz^2 n + m))^2$. In particular, for the MSE loss and an operator such that $\LFz$ is of constant order (as is the case when $\fopnl$ is linear), we only need $k \gtrsim n^{2}m$. The main reason underlying this improvement is that there is no need in this case to control the deviation of $\Vv$ from its initial point to compute the local Lipschitz constant of the jacobian of the network. This allows to have a far better Lipschitz constant estimate which turns out to be even global in this case.
\end{remark}
\begin{remark}[Effect of input dimension $d$]
    Finally, the dependence on $d$ is far smaller (by a log factor) than the one on $n$ and $m$. In the way we presented the theorem, it does also affect the probability obtained but it is possible to write the same probability without $d$ and with a stronger impact of $n$. This indicates that $d$ plays a very minor role on the overparametrization level whereas $k$ is the key to reaching the overparametrized regime we are looking for. In fact, this is demonstrated by our numerical experiments where we obtained the same results by using very small $d \in [1,10]$ or larger values up to 500, for all our experiments with potentially large $n$.
\end{remark}

\subsection{Proofs}
We start with the following lemmas that will be instrumental in the proof of Theorem~\ref{th:dip_two_layers_converge}.

\begin{lemma}[Bound on $\sigminjgz$ with both layers trained]
\label{lemma:min_eigenvalue_singvalue_init_both_layer}
Consider the one-hidden layer network \eqref{eq:dipntk} with both layers trained under assumptions \ref{ass:phi_diff} and \ref{ass:u_sphere}-\ref{ass:v_init}. We have
\begin{align*}
\sigminjgz \geq \sqrt{\Cphi^2 + \Cphid^2}/2 
\end{align*}
with probability at least $1-2n^{-1}$ provided that $k/\log(k) \geq C n\log(n)$ for $C > 0$ large enough that depends only on $B$, $\Cphi$, $\Cphid$ and $D$.
\end{lemma}

\begin{proof}

Define the matrix $\Hv = \jthetaz\jthetaz\tp$. Since $\uv$ is on the unit sphere, $\Hv$ reads
\begin{align*}
    \Hv = \frac{1}{k} \sum_{i=1}^k \Hv_i, \qwhereq \Hv_i \eqdef \phi'(\Wv^i(0)\uv)^2\Vv_i(0)\Vv_i(0)\tp + \phi(\Wv^i(0)\uv)^2\Id_n .
\end{align*}
It then follows that
\begin{align*}
    \Expect{}{\Hv} &= \frac{1}{k}\Expect{X\sim \stddistrib}{\phi'(X)^2} \sum_{i=1}^k\Expect{}{\Vv_i(0) \Vv_i(0)\tp} + \Expect{X\sim \stddistrib}{\phi(X)^2}\Id_n\\
    &= (\Cphid^2 + \Cphi^2) \Id_n,
\end{align*}
where we used \ref{ass:u_sphere}, \ref{ass:w_init} and orthogonal invariance of the Gaussian distribution, hence $\Wv^i(0)\uv$ are iid in $\stddistrib$, as well as \ref{ass:v_init} and independence between $\Vv(0)$ and $\Wv(0)$. Moreover, $\Expect{}{\phi(X)} \leq \Cphi$, and since $X\sim\stddistrib$ and in view of \ref{ass:phi_diff}, we can upper-bound $\phi(X)$ using the Gaussian concentration inequality to get
\begin{align}
    \prob{\phi(X) \geq \Cphi\sqrt{\log(nk)} + \tau} \leq \prob{\phi(X) \geq \Expect{}{\phi(X)} + \tau}
    \leq \exp{\left(-\frac{\tau^2}{2B^2}\right)}.
\end{align}
By choosing $\tau = \sqrt{2}B\sqrt{\log(nk)}$, and taking $c_1 = \Cphi + \sqrt{2}B$, we get
\begin{align}
    \prob{\phi(X) \geq c_1\sqrt{\log(nk)}} \leq (nk)^{-1} .
\end{align}
Using a union bound, we obtain
\begin{align*}
\prob{\max_{i \in [k]} \phi(\Wv^i(0)\uv)^2 > c_1\log(nk)} \leq n(nk)^{-1} \leq n^{-1} .
\end{align*}
Thus, with probability at least $1-n^{-1}$ we get
\begin{align*}
\max_{i \in [k]} \lmax\pa{\Hv_i} 
\leq B^2D^2n + c_1\log(nk)
\leq c_2n\log(k) , 
\end{align*}
where $c_2 = B^2D^2 + 2c_1$.
We can then apply the matrix Chernoff inequality \cite[Theorem~5.1.1]{tropp_introduction_2015} to get
\begin{align*}
&\prob{\sigminjgz \leq \delta\sqrt{\Cphid^2 + \Cphi^2}} \\
&\leq \prob{\sigminjgz \leq \delta\sqrt{\Cphid^2 + \Cphi^2} \; \bigg| \max_{i \in [k]} \lmax\pa{\Hv_i} \leq c_2n\log(k)} \\
& \quad + \prob{\max_{i \in [k]} \lmax\pa{\Hv_i} \geq c_2n\log(k)} \\
&\leq ne^{-\frac{(1-\delta)^2k(\Cphid^2 + \Cphi^2)}{c_2n\log(k)}} + n^{-1} .
\end{align*}
Taking $\delta=1/2$ and $k$ as prescribed with a sufficiently large constant $C$, we conclude.

\end{proof}



\begin{lemma}[Local Lipschitz constant of $\jcal$ with both layers trained]\label{lemma:lip-Jacobian-both-layers}
Suppose that assumptions \ref{ass:phi_diff}, \ref{ass:u_sphere} and \ref{ass:v_init} are satisfied. For the one-hidden layer network \eqref{eq:dipntk} with both layers trained, we have for $n \geq 2$ and any $\rho > 0$:
\[
\Lip_{\Ball(\thetavz,\rho)}(\jcal) \leq B (1+2(D+\rho))\sqrt{\frac{n}{k}} .
\]
\end{lemma}

\begin{proof}
Let $\thetav \in \R^{k(d + n)}$ (resp. $\thetavalt$) be the vectorized form of the parameters of the network $(\Wv,\Vv)$ (resp. $(\Wvalt,\Vvalt)$). Using the expression of the Jacobian $\Jg$, we have, for $\thetav, \thetavalt \in \Ball(R,\thetavz)$,
 
\begin{align*}
&\norm{\jtheta - \jcal(\thetavalt)}^2 \\
&\leq \frac{1}{k} \pa{\sum_{i=1}^k \normf{\phi'(\Wv^i\uv)\Vv_i\uv\tp - \phi'(\Wvalt^i\uv)\Vvalt_i\uv\tp}^2 + 
n\norm{\phiWu - \phiWualt}^2} \\
&\leq \frac{1}{k} \Bigg(2 \sum_{i=1}^k \pa{\normf{\phi'(\Wv^i\uv)\pa{\Vv_i - \Vvalt_i}\uv\tp}^2 + \normf{\pa{\phi'(\Wv^i\uv)-\phi'(\Wvalt^i\uv)}\Vvalt_i\uv\tp}^2} \\
&\qquad + n\norm{\phiWu - \phiWualt}^2\Bigg) \\ 
&\leq \frac{1}{k} \pa{2 B^2 \sum_{i=1}^k \pa{\norm{\Vv_i - \Vvalt_i}^2 + \norm{\Wv^i-\Wvalt^i}^2\norm{\Vvalt_i}^2} + n\norm{\phiWu - \phiWualt}^2}\\
&\leq \frac{1}{k} \pa{2 B^2 \normf{\Vv - \Vvalt}^2 + 2B^2\sum_{i=1}^k\norm{\Wv^i-\Wvalt^i}^2\norm{\Vvalt_i}^2 +  B^2n\norm{(\Wv - \Wvalt)\uv}^2}\\
&\leq \frac{1}{k} \pa{2 B^2 \normf{\Vv - \Vvalt}^2 + 2B^2\max_i\norm{\Vvalt_i}^2\normf{\Wv-\Wvalt}^2 +  B^2n\normf{\Wv - \Wvalt}^2}\\
&\leq  \frac{n}{k}B^2\pa{\normf{\Vv - \Vvalt}^2 + \normf{\Wv - \Wvalt}^2} + \frac{2}{k}B^2\max_i\norm{\Vvalt_i}^2\normf{\Wv-\Wvalt}^2\\
&=  \frac{n}{k} B^2 \norm{\thetav - \thetavalt}^2 + \frac{2}{k}B^2\max_i\norm{\Vvalt_i}^2\normf{\Wv-\Wvalt}^2 .
\end{align*}
Moreover, for any $i \in [k]$:
\[
\norm{\Vvalt_i}^2 \leq 2\norm{\Vv_i(0)}^2 + 2 \norm{\Vvalt_i-\Vv_i(0)}^2 \leq 2\norm{\Vv_i(0)}^2 + 2 \norm{\thetav-\thetavz}^2 \leq 2nD^2 + 2\rho^2 ,
\]
where we used \ref{ass:v_init}.
Thus
\begin{align*}
\norm{\jtheta - \jcal(\thetavalt)}^2 \leq \frac{n}{k} B^2\pa{1+4D^2+2\rho^2}\norm{\thetav - \thetavalt}^2 .
\end{align*}
Taking the square-root and using that $(a+b)^{1/2} \leq a^{1/2} + b^{1/2}$ for any $a, b \geq 0$, we conclude.
\end{proof}



\begin{lemma}[Bound on the initial error]\label{lemma:bound_initial_misfit}
Under assumptions \ref{ass:phi_diff}, \ref{ass:F_diff} and \ref{ass:u_sphere} to \ref{ass:v_init}, the initial error of the network satisfies
\begin{align*}
\norm{\yv(0) - \yv} \leq C\LFz \sqrt{n\log(d)} + \sqrt{m}\pa{\norminf{\fopnl(\xvc)} + \norminf{\veps}} ,
\end{align*}
with probability at least $1 - d^{-1}$, where $C$ is a constant that depends only on $B$, $\Cphi$, and $D$.
\end{lemma}

\begin{proof}
By \ref{ass:F_diff} and the mean value theorem, we have
\[
\norm{\yvz - \yv} \leq \max_{\xv \in \Ball(0,\norm{\xvz})} \norm{\Jf(\xv)}\norm{\xvz} + \sqrt{m} \pa{\norminf{\fopnl(\xvc)} + \norminf{\veps}} ,
\]
where $\xvz = \gdipz = \frac{1}{\sqrt{k}} \sum_{i=1}^k \phi(\Wv^i(0)\uv)\Vv_i(0)$. Moreover, by \ref{ass:v_init}:
\[
\norm{\gdipz} \leq \max_i \norm{\Vv_i(0)} \frac{1}{\sqrt{k}} \sum_{i=1}^k \abs{\phi(\Wv^i(0)\uv)} \leq D\sqrt{n} \frac{1}{\sqrt{k}} \sum_{i=1}^k \abs{\phi(\Wv^i(0)\uv)} .
\]

We now prove that the last term concentrates around its expectation. First, owing to \ref{ass:u_sphere} and \ref{ass:w_init}, we can argue using orthogonal invariance of the Gaussian distribution and independence to infer that
\[
\Expect{}{\frac{1}{\sqrt{k}} \sum_{i=1}^k \abs{\phi(\Wv^i(0)\uv)}}^2 \leq \frac{1}{k}\Expect{}{\pa{\sum_{i=1}^k \abs{\phi(\Wv^i(0)\uv)}}^2} = \Expect{}{\phi(\Wv^1(0)\uv)^2} = \Cphi^2 .
\]
In addition, the triangle inequality and Lipschitz continuity of $\phi$ (see \ref{ass:phi_diff}) yields
\begin{align*}
\frac{1}{\sqrt{k}}\abs{\sum_{i=1}^k \abs{\phi(\Wv^i\uv)} - \abs{\phi(\Wvalt^i\uv)}}
&\leq \frac{1}{\sqrt{k}}\sum_{i=1}^k \abs{\phi(\Wv^i\uv) - \phi(\Wvalt^i\uv)} \\
&\leq B\pa{\frac{1}{\sqrt{k}}\sum_{i=1}^k \norm{\Wv^i - \Wvalt^i}} \leq B\normf{\Wv - \Wvalt} .
\end{align*}
We then get using the Gaussian concentration inequality that
\begin{align*}
&\prob{\frac{1}{\sqrt{k}} \sum_{i=1}^k \abs{\phi(\Wv^i(0)\uv)} \geq \Cphi \sqrt{\log(d)} + \tau} \\
&\leq \prob{\frac{1}{\sqrt{k}} \sum_{i=1}^k \abs{\phi(\Wv^i(0)\uv)} \geq \Expect{}{\frac{1}{\sqrt{k}} \sum_{i=1}^k \abs{\phi(\Wv^i(0)\uv)}} + \tau}
\leq e^{-\frac{\tau^2}{2B^2}} .
\end{align*}
Taking $\tau = \sqrt{2}B \sqrt{\log(d)}$, we get
\[
\norm{\xvz} \leq C \sqrt{n\log(d)} 
\] 
with probability at least $1-d^{-1}$. Since the event above implies $\Ball(0,\norm{\xvz}) \subset \Ball\pa{0,C\sqrt{n\log(d)}}$, we conclude.
\end{proof}

\begin{proof}[Proof of Theorem~\ref{th:dip_two_layers_converge}]
Proving Theorem~\ref{th:dip_two_layers_converge} amounts to showing that \eqref{eq:bndR} holds with high probability under our scaling. This will be achieved by combining Lemma~\ref{lemma:min_eigenvalue_singvalue_init_both_layer}, Lemma~\ref{lemma:lip-Jacobian-both-layers} and Lemma~\ref{lemma:bound_initial_misfit} as well as the union bound.

From Lemma~\ref{lemma:min_eigenvalue_singvalue_init_both_layer}, we have
\begin{align*}
\sigminjgz \geq \sqrt{\Cphi^2 + \Cphid^2}/2 
\end{align*}
with probability at least $1-2n^{-1}$ provided $k \geq C_0 n\log(n)\log(k)$ for $C_0 > 0$. On the other hand, from Lemma~\ref{lemma:lip-Jacobian-both-layers}, and recalling $R$ from \eqref{eq:RandR'}, we have that $R$ must obey
\[
R \geq \frac{\sigminjgz}{2B((1+2D)+2R)}\sqrt{\frac{k}{n}} \geq \frac{\sqrt{\Cphi^2 + \Cphid^2}}{8B((1/2+D)+R)}\sqrt{\frac{k}{n}} .
\]
Solving for $R$, we arrive at
\[
R \geq \frac{\sqrt{(1/2+D)^2 + \frac{\sqrt{(\Cphi^2 + \Cphid^2)\frac{k}{n}}}{2B}}-(1/2+D)}{2}.
\]
Simple algebraic computations and standard bounds on $\sqrt{1+a}$ for $a \in [0,1]$ show that
\[
R \geq C_1\pa{\frac{k}{n}}^{1/4}
\]
whenever $k \gtrsim n$, $C_1$ being a positive constant that depends only on $B$, $\Cphi$, $\Cphid$ and $D$.

Thanks to \ref{ass:l_smooth} and \ref{ass:min_l_zero}, we have by the descent lemma, see e.g. \cite[Lemma~2.64]{BauschkeBook}, that
\[
\lossyzy \leq \max_{\vv \in [\yv,\yvz]}\frac{\norm{\nabla \lossy(\vv)}}{\norm{\vv-\yv}} \frac{\norm{\yvz-\yv}^2}{2} .
\] 
Combining Lemma~\ref{lemma:bound_initial_misfit} and the fact that 
\[
[\yv,\yvz] \subset \Ball(0,\norm{\yv}+\norm{\yvz}) 
\]
then allows to deduce that with probability at least $1-d^{-1}$, we have
\[
\lossyzy \leq \frac{\LLz}{2}\pa{C\LFz \sqrt{n\log(d)} + \sqrt{m}\pa{\norminf{\fopnl(\xvc)} + \norminf{\veps}}}^2 .
\] 
Therefore, using the union bound and the fact that $\psi$ is increasing, it is sufficient for \eqref{eq:bndR} to be fulfilled with probability at least $1-2n^{-1}-d^{-1}$, that
\begin{equation}\label{eq:R'Rbnd}
\frac{4}{\sigminF\sqrt{\Cphi^2 + \Cphid^2}}\psi\pa{\frac{\LLz}{2}\pa{C\LFz \sqrt{n\log(d)} + \sqrt{m}\pa{\norminf{\fopnl(\xvc)} + \norminf{\veps}}}^2} < C_1\pa{\frac{k}{n}}^{1/4} ,
\end{equation}
whence we deduce the claimed scaling.
\end{proof}

%% file: tex/sec_Numerical.tex

\section{Numerical Experiments}\label{sec:expes}



To validate our theoretical findings, we carried out a series of experiments on two-layer neural networks in the DIP setting. Therein, 25000 gradient descent iterations with a fixed step-size were performed. If the loss reached a value smaller than $10^{-7}$, we stopped the training and considered it has converged. For these networks, we only  trained the first layer, $\Wv$, and fixed the second layer, $\Vv$, as it allows to have better theoretical scalings as discussed in Remark~\ref{rq:V_fixed}. Every network was initialized with respect to the assumption of this work where we used sigmoid activation function. The entries of $\xvc$ are drawn from $\stddistrib$ while the entries of the linear forward operator $\fopnl$ are drawn from $\mathcal{N}(0,1/\sqrt{n})$ to ensure that $\LFz$ is of constant order. 

Our first experiment in Figure~\ref{fig:plots_heatmaps} studies the convergence to a zero-loss solution of networks with different architecture parameters in a noise-free context. The absence of noise allows the networks to converge faster which is helpful to check convergence in 25000 iterations. We used $\lossy(\yvt) = \frac{1}{2}\norm{\yvt - \yv}^2$ as it should gives good exponential decay. For each set of architecture parameters, we did 50 runs and calculated the frequency at which the network arrived at the error threshold of $10^{-7}$. We present two experiments, in the first one we fix $m=10$ and $d=500$ and let $k$ and $n$ vary while in the second we fix $n=60$, $d=500$ and we let $k$ and $m$ vary.


\begin{figure}[htb!]
    \begin{subfigure}{.5\textwidth}
        \centering
        \includegraphics[width=.9\linewidth]{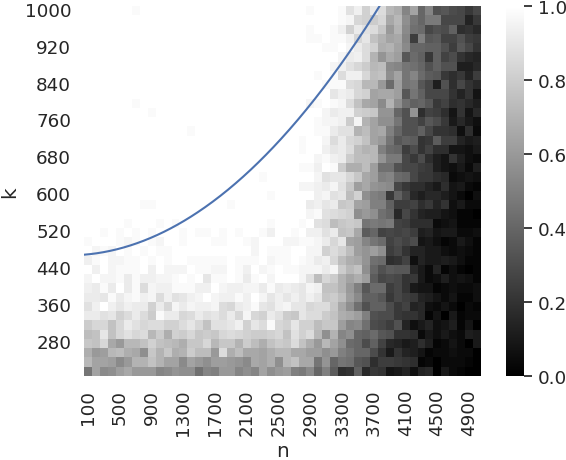}
        \caption{$k$ vs $n$}
        \label{fig:heat_s1}
    \end{subfigure}%
    \begin{subfigure}{.5\textwidth}
        \centering
        \includegraphics[width=.9\linewidth]{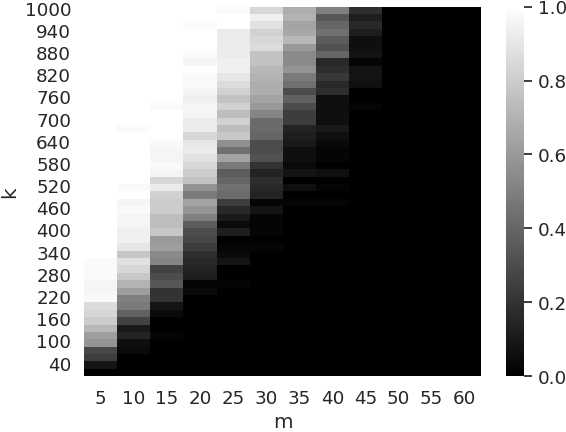}
        \caption{$k$ vs $m$}
        \label{fig:heat_s2}
    \end{subfigure}%
    \caption{Probability of converging to a zero-loss solution for networks with different architecture parameters confirming our theoretical predictions: linear dependency between $k$ and $m$ and at least quadratic dependency between $k$ and $n$. The blue line is a quadratic function representing the phase transition fitted on the data.}
    \label{fig:plots_heatmaps}
\end{figure}

Based on Remark~\ref{rq:V_fixed} concerning Theorem~\ref{th:dip_one_layer_fixed} which is a specialisation of Theorem~\ref{th:dip_two_layers_converge}, for our experimental setting (MSE loss with $\LFz$ of constant order), one should expect to observe convergence to zero-loss solutions when $k\gtrsim n^{2} m$. We observe in Figure~\ref{fig:heat_s1} the relationship between $k$ and $n$ for a fixed $m$. In this setup where $n \gg m$ and $\fop$ is Gaussian, we expect a quadratic relationship which seems to be the case in the plot. It is however surprising that with values of $k$ restricted to the range $[20,1000]$, the network converges to zero-loss solution with high probability for situations where $n > k$ which goes against our intuition for these underparametrized cases.

Additionally, the observation of Figure~\ref{fig:heat_s2} provides a very different picture when the ratio $m/n$ goes away from 0. We first see clearly the expected linear relationship between $k$ and $m$. However, we used in this experiment $n=60$ and we can see that for the same range of values of $k$, the method has much more difficulty to converge with already small $m$. This indicates that the ratio $m/n$ plays an important role in the level of overparametrization necessary for the network to converge. It is clear from these results that our bounds are not tight as we observe convergence for lower values of $k$ than expected.


In our second experiment presented in Figure~\ref{fig:signal_convergence}, we look at the signal evolution under different noise levels when the restricted injectivity constraint \ref{ass:F_inj} is met to verify our theoretical bound on the signal loss. Due to the fact that our networks can span the entirety of the space $\R^n$, this injectivity constraint becomes a global one, which forces us to use a square matrix as our forward operator, we thus chose to use $n=m=10$. Following the discussion about assumption \ref{ass:nablaL_F}, we choose to use $\lossy(\yvt) = \eta(\norm{\yvt - \yv}^2)$ with $\eta(s) = s^{p+1}/(2(p+1))$ where $p\in[0,1]$ with $p=0.2$ for this specific experiment. We generated once a forward operator with singular values in $\{\frac{1}{z^2 +1} \mid z\in [0,9]\}$ and kept the same one for all the runs. To better see the convergence of the signal, we ran these experiments for 200000 iterations. Furthermore $\epsilon$ is a noise vector with entries drawn from a uniform distribution $U(-\beta,\beta)$ with $\beta$ representing the level of noise.

\begin{figure}[htb!]
    \begin{subfigure}[t]{.48\textwidth}
        \centering
        \includegraphics[width=\linewidth]{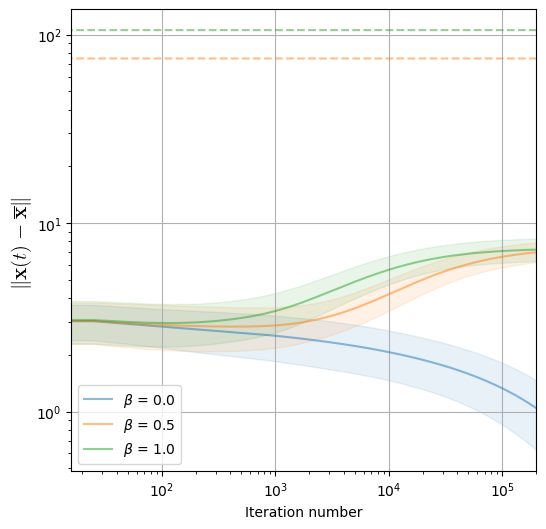}
        \caption{Signal distance to $\xvc$ for different noise levels. The mean and standard deviation of 50 runs are plotted. The dashed line represents the expectation of the theoretical upper bound of this distance when $t \to +\infty$.}
        \label{fig:signal_convergence}
    \end{subfigure}%
    \hspace{0.04\textwidth}
    \begin{subfigure}[t]{.48\textwidth}
        \centering
        \includegraphics[width=\linewidth]{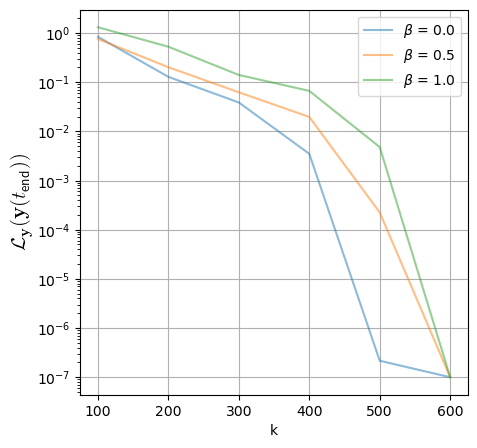}
        \caption{Loss found at time $t_{\text{end}}$, which correspond to the end of the optimization process for networks with varying number of neurons $k$ with three noise levels. Each point is averaged from 50 runs.}
        \label{fig:noise_vs_k}
    \end{subfigure}%
    \caption{Effect of the noise on both the signal and the loss convergence in different contexts.}
    \label{fig:noise_effect}
\end{figure}


In this figure, we plot the mean and the standard deviation of 50 runs for each noise level. For comparison we also show with the dashed line the expectation of the theoretical upper bound, corresponding to $\Expect{}{\norm{\varepsilon}/\muf} \geq \frac{\sqrt{m\beta}}{\sqrt{6}\muf}$. 
We observe that the gap between this theoretical bound and the mean of the signal loss is growing as the noise level grows. This indicates that the more noise, the less tighter our bound becomes.
We also see different convergence profiles of the signal depending on the noise level which is to be expected as the network will fit this noise to optimize its loss. Of course, when there is no noise, the signal tends to the ground truth thanks to the injectivity of the forward operator.


We continue the study of the effect of the noise on the convergence of the networks in Figure~\ref{fig:noise_vs_k}. We show the convergence profile of the loss depending on the noise level and $k$. For that we fixed $n=1000$, $m=10$, $d=10$, $p=0.1$ and ran the optimization of networks with different $k$ and $\beta$ values and we took the loss value obtained at the end of the optimization. The results are averaged from 50 runs and help to see that even if a network with insufficient overparametrization does not converge to a zero-loss solution, the more neurons it has, the better in average the solution in term of loss value. Moreover, this effect seems to stay true even with noise. It is interesting to see the behavior of the loss in such cases that are not treated by our theoretical framework.


For our fourth experiment, we are interested by the effect on the convergence speed of the parameter $p$ of the loss previously described. We fixed $n=1000$, $m=10$ and $k=800$ and varied $p$ between 0 and 1. For each choice of $p$, we trained 50 networks and show the mean value of the loss at each iteration in Figure~\ref{fig:p_convergence}. We chose to use $10^{6}$ iteration steps and let the optimization reach a limit of $10^{-14}$. As expected by corollary~\ref{col:loja_rate}, smaller $p$ values lead to faster convergence rate in general. Indeed, smaller $p$ values are closer to the case where $\alpha=1/2$ in the corollary and higher $p$ values means that $\alpha$ will grow away from $1/2$ which worsens the theoretical rate of convergence.

\begin{figure}[!htb]
    \centering
    \includegraphics[width=\linewidth]{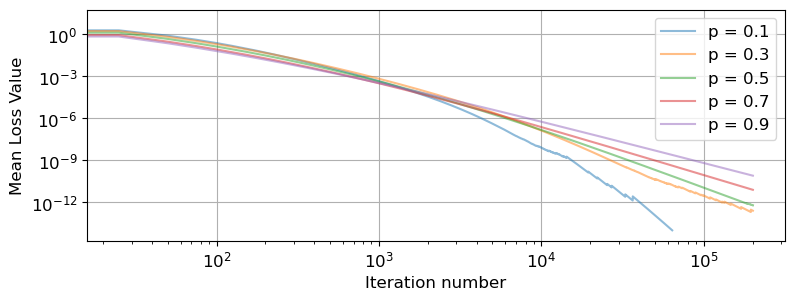}
    \caption{Convergence profile of different losses parametrized by $p$. The mean loss values at each iteration of 50 networks are plotted.}
    \label{fig:p_convergence}
\end{figure}

\section{Conclusion and Future Work}\label{sec:conclu}

This paper studied the optimization trajectories of neural networks in the inverse problem setting and provided both convergence guarantees for the network and recovery guarantees of the solution. Our results hold for a broad class of loss functions thanks to the Kurdyka-\L ojasewiecz inequality.  We also demonstrate that for a two-layers DIP network with smooth activation and sufficient overparametrization, we obtain with high probability our theoretical guarantees. Our proof relies on bounding the minimum singular values of the Jacobian of the network through an overparametrization that ensures a good initialization of the network. Then the recovery guarantees are obtained by decomposing the distance to the signal in different error terms explained by the noise, the optimization and the architecture. Although our bounds are not tight as demonstrated by the numerical experiments, they provide a step towards the theoretical understanding of neural networks for inverse problem resolution. In the future we would like to study more thorougly the multilayer case and adapt our result to take into account the ReLU function. Another future direction is to adapt our analysis to the supervised setting and to provide a similar analysis with accelerated optimization methods.


%% file: tex/appendix.tex
\section{Reconstruction Bound in High Signal/Noise Ratio }\label{appendix:alternative_rec_bound}

The reconstruction bound given in \eqref{eq:xrate} relies on assumption~\ref{ass:F_inj} which requires injectivity of $\fopnl$ on $\Sigma$. An alternative way of deriving a similar bound only requires to impose restricted injectivity of the jacobian of $\fopnl$ at one point. The trade-off is that it is only valid for low noise level.

\begin{theorem}\label{thm:alternative_rec_bound}
    Under the setting of Theorem~\ref{thm:main}, let $\loss$ be convex and $\Argmin(\lossy(\cdot)) = \{\yv\}$. Assume that 
    \begin{assumption}\label{ass:F_jac_inj}
     $\ker(\Jf(\xvcsigma)) \cap T_{\Sigma'}(\xvcsigma) = {0}$.
    \end{assumption}
    Denote 
    \[
    \LF \eqdef \max_{\xv \in \Ball(0,2\norm{\xvc})} \norm{\Jf(\xv)} < +\infty . \quad \text{ and } \quad \delta(t) \eqdef \frac{2\psi\pa{\Psi\inv\pa{\gamma(t)}}}{\sigminjgz\sigminF} . 
    \]
    Then for $\dist(\xvc,\Sigma')$ and $\norm{\varepsilon}$ small enough and all $t > 0$ sufficiently large, we have
    \begin{align}
    \norm{\xvt - \xvc} \leq 2\pa{\muf\inv(\delta(t) + \norm{\varepsilon}) + \frac{\LJf}{\muf}\dist(\xvc,\Sigma')^2 + \left(1+\frac{\LF}{\muf}\right)\dist(\xvc,\Sigma')} ,
    \end{align}
    where $\LJf > 0$ is a local Lispchitz constant of $\Jf$ on $\Ball\pa{0,\rho+2\norm{\xvc}}$.

\end{theorem}

\begin{proof}
    Observe that Assumption~\ref{ass:F_jac_inj} implies that
    $\displaystyle{\muf = \lmin(\Jf(\xvcsigma);T_{\Sigma}(\xvcsigma))>0}$. Thus we have
    \begin{align*}
        \norm{\xvt - \xvc} &\leq \norm{\xvt - \xvcsigma} + \dist(\xvc,\Sigma')\\
        &\leq \muf\inv \norm{\Jf(\xvcsigma)(\xvt - \xvcsigma)} + \dist(\xvc,\Sigma').
    \end{align*}
    Recall from Theorem~\ref{thm:main}\ref{thm:main_y_bounded} that $\theta(\cdot)$ is bounded, and therefore so is $\xv(\cdot)$ by continuity of $\gv(\uv,\cdot)$; i.e. $\xv(t) \in \Ball\pa{0,\rho}$ for some $\rho > 0$. It then follows from the local Lipschitz continuity assumption on $\Jf$ in \ref{ass:F_diff} that there exists $\LJf > 0$ such that for all $\zv, \zv^\prime \in \Ball\pa{0,\rho+2\norm{\xvc}}$
    \[
    \norm{\Jf(\zv)-\Jf(\zv^\prime)} \leq \LJf \norm{\zv-\zv^\prime} .
    \]
    In turn, we have
    \begin{align*}
    &\norm{\fopnl(\xvt)-\fopnl(\xvcsigma) - \Jf(\xvcsigma)(\xvt - \xvcsigma)}\\
    &= \norm{\int_0^1 \left(\Jf\left(\xvcsigma+t(\xvt - \xvcsigma)\right)-\Jf(\xvcsigma)\right)(\xvt - \xvcsigma) dt} \\
    &\leq \frac{\LJf}{2}\norm{\xvt - \xvc}^2 .
    \end{align*}
    Thus    
    \begin{align*}
         \norm{\xvt - \xvc} 
         &\leq \muf\inv\left(\norm{\fopnl(\xvt) - \fopnl(\xvcsigma)} + \frac{\LJf}{2}\norm{\xvt - \xvcsigma}^2\right) + \dist(\xvc,\Sigma)\\
         &\leq \muf\inv(\norm{\yvt - \yv} + \norm{\fopnl({\xvcsigma}) - \yv} + \frac{\LJf}{2}\norm{\xvt - \xvcsigma}^2) + \dist(\xvc,\Sigma).
    \end{align*}
    From~\eqref{eq:klerrbnd}, we get
    \begin{multline*}
        \norm{\xvt - \xvc} \\
        \leq \muf\inv\pa{\delta(t) + \norm{\varepsilon} + \norm{\fopnl({\xvcsigma}) - \fopnl(\xvc)} + \LJf\pa{\norm{\xvt - \xvc}^2 +  \dist(\xvc,\Sigma)^2}} \\
        + \dist(\xvc,\Sigma).
    \end{multline*}
    Using~\eqref{eq:diff_Fx_Fxsigma}, we obtain the following inequality of a quadratic form
    \begin{multline*}
        -\frac{\LJf}{\muf}\norm{\xvt - \xvc}^2 + \norm{\xvt - \xvc} \\ 
        - \muf\inv(\delta(t) + \norm{\varepsilon}) - \frac{\LJf}{\muf}\dist(\xvc,\Sigma)^2 - \left(1+\frac{\LF}{\muf}\right)\dist(\xvc,\Sigma) \leq 0.
    \end{multline*}
    Since $\delta(t) \to 0$, there exists $\tilde{t} > 0$ such that $\delta(t)$ is small enough for all $t \geq \tilde{t}$. Thus for all such $t$ and for sufficiently small $\dist(\xvc,\Sigma')$ and $\norm{\varepsilon}$, we know that the above polynomial has two real positive roots. Solving for $\norm{\xvt - \xvc}$, we get for $\dist(\xvc,\Sigma')$ and $\norm{\varepsilon}$ small enough and $t \geq \tilde{t}$, that
    \begin{align*}
        &\norm{\xvt-\xvc} 
        \leq \frac{\muf}{2\LJf}\\
        &\pa{1 - \sqrt{1 - 4\frac{\LJf}{\muf}\pa{(\muf\inv(\delta(t)+\norm{\varepsilon}) + \frac{\LJf}{\muf}\dist(\xvc,\Sigma')^2 + \left(1+\frac{\LF}{\muf}\right)\dist(\xvc,\Sigma')}}}\\
        &\leq 2\pa{\muf\inv(\delta(t) + \norm{\varepsilon}) + \frac{\LJf}{\muf}\dist(\xvc,\Sigma')^2 + \left(1+\frac{\LF}{\muf}\right)\dist(\xvc,\Sigma')} .
    \end{align*}
\end{proof}

\section{Overparametrization Bound When the Linear Layer is Fixed}\label{appendix:V_fixed}

In the setting described in Section~\ref{sec:dip}, if one fixes the linear layer, as is usually done in the literature, a better overparametrization bound can be derived.

\begin{theorem}\label{th:dip_one_layer_fixed}
Under the setting of Theorem~\ref{th:dip_two_layers_converge} where
\[
\LFz = \max_{\xv \in \Ball\pa{0,C\sqrt{n\log(d)}}} \norm{\Jf(\xv)}
\]
and
\[
\LLz = \max_{\vv \in \Ball\pa{0,C\LFz\sqrt{n\log(d)}+\sqrt{m}\pa{\norminf{\fopnl(\xvc)}+\norminf{\veps}}}} \frac{\norm{\nabla_\vv \lossy(\vv)}}{\norm{\vv-\yv}},
\]
consider the one-hidden layer network \eqref{eq:dipntk} where only the first layer is trained with the initialization satisfying \ref{ass:u_sphere}-\ref{ass:v_init} and the architecture parameters obeying
\begin{align*}
k \geq C' \sigminF^{-2} n \psi\pa{\frac{\LLz}{2}\pa{C\LFz \sqrt{n\log(d)} + \sqrt{m}\pa{\norminf{\fopnl(\xvc)} + \norminf{\veps}}}^2}^2 .
\end{align*}
Then \eqref{eq:bndR} holds with probability at least $1 - n^{-1} - d^{-1}$, where $C$ and $C'$ are positive constants that depend only on the activation function and $D$.
\end{theorem}


\begin{proof}
The proof follows a very similar pattern as in the case where both layers are trained. The two main changes happen in the bounds on $\sigmin(\jthetaz)$ and $\Lip(\jcal)$. First, the constant on $\sigmin(\jthetaz)$ changes slightly but is still in $O(1)$ as described in lemma~\ref{lemma:min_eigenvalue_singvalue_init_one_layer}. The main change from the previous setting is that $\Lip(\jcal)$ is now a global constant given in lemma~\ref{lemma:lip-Jacobian-one-layer}. We now follow the same structure as in the proof of theorem~\ref{th:dip_two_layers_converge} and see that by Lemma~\ref{lemma:min_eigenvalue_singvalue_init_one_layer} and Lemma~\ref{lemma:lip-Jacobian-one-layer} we have that
\begin{align*}
    R \geq \frac{\Cphid}{2BD}\sqrt{\frac{k}{n}} \quad \text{thus,} \quad R \geq C_1\left(\frac{k}{n}\right)^{1/2}.
\end{align*}

Moreover, let us recall that by combining lemma~\ref{lemma:bound_initial_misfit} and the fact that
\[
[\yv,\yvz] \subset \Ball(0,\norm{\yv}+\norm{\yvz}) 
\]
we can deduce that with probability at least $1-n^{-1}$,
\[
\lossyzy \leq \frac{\LLz}{2}\pa{C\LFz \sqrt{n\log(d)} + \sqrt{m}\pa{\norminf{\fopnl(\xvc)} + \norminf{\veps}}}^2 .
\] 
Therefore, by using a union bound and that $\psi$ is increasing, for~\eqref{eq:bndR} to hold with probability $1 - n^{-1}-d^{-1}$, we need
\begin{align*}
    \sigminF^{-1} \psi\left(\frac{\LLz}{2}\pa{C\LFz \sqrt{n\log(d)} + \sqrt{m}\pa{\norminf{\fopnl(\xvc)} + \norminf{\veps}}}^2\right) \leq C_1\left(\frac{k}{n}\right)^{1/2}
\end{align*}
which gives the claim.
\end{proof}

\begin{lemma}[Bound on $\sigmin(\jthetaz)$]
\label{lemma:min_eigenvalue_singvalue_init_one_layer}
For the one-hidden layer network \eqref{eq:dipntk}, under assumptions \ref{ass:phi_diff} and \ref{ass:u_sphere} to \ref{ass:v_init}, we have
\begin{align*}
\sigmin(\jthetaz) \geq \Cphid/2 
\end{align*}
with probability at least $1-n^{-1}$ provided $k \geq C n\log(n)$ for $C > 0$ large enough that depends only on $\phi$ and the bound on the  entries of $\Vv$.
\end{lemma}

\begin{proof}
Define the matrix $\Hv = \jthetaz\jthetaz\tp$. For the two-layer network, and since $\uv$ is on the unit sphere, $\Hv$ reads
\[
\Hv = \frac{1}{k} \sum_{i=1}^k \phi'(\Wv^i(0)\uv)^2\Vv_i\Vv_i\tp .
\]
It follows that
\[
\Expect{}{\Hv} = \Expect{X\sim \stddistrib}{\phi'(X)^2}\frac{1}{k} \sum_{i=1}^k\Expect{}{\Vv_i \Vv_i\tp} = \Cphid^2 \Id_n ,
\]
where we used \ref{ass:u_sphere}-\ref{ass:w_init} and orthogonal invariance of the Gaussian distribution, hence $\Wv^i(0)\uv$ are iid $\stddistrib$, as well as \ref{ass:v_init} and independence between $\Vv$ and $\Wv(0)$. Moreover,
\[
\lammax(\phi'(\Wv^i(0)\uv)^2\Vv_i\Vv_i\tp) \leq B^2 D^2 n .
\]
We can then apply the matrix Chernoff inequality \cite[Theorem~5.1.1]{tropp_introduction_2015} to get
\[
\prob{\sigmin(\jthetaz) \leq \delta\Cphid} \leq ne^{-\frac{(1-\delta)^2k\Cphid^2}{2B^2 D^2 n}} .
\]
Taking $\delta=1/2$ and $k$ as prescribed, we conclude.
\end{proof}

\begin{lemma}[Global Lipschitz constant of $\jcal$ with linear layer fixed]\label{lemma:lip-Jacobian-one-layer}
Suppose that assumptions \eqref{ass:phi_diff}, \eqref{ass:u_sphere} and \eqref{ass:v_init} are satisfied. For the one-hidden layer network \eqref{eq:dipntk} with only the hidden-layer trained, we have
\[
\Lip(\jcal) \leq BD \sqrt{\frac{n}{k}} .
\]
\end{lemma}

\begin{proof}
We have for all $\Wv, \Wvalt \in \R^{k \times d}$, 
\begin{align*}
\norm{\jW - \jWalt}^2 
&\leq \frac{1}{k} \sum_{i=1}^k |\phi'(\Wv^i\uv) - \phi'(\Wvalt^i\uv)|^2 \normf{\Vv_i\uv\tp}^2 \\
&=\frac{1}{k} \sum_{i=1}^k |\phi'(\Wv^i\uv) - \phi'(\Wvalt^i\uv)|^2 \norm{\Vv_i}^2 \\
&\leq B^2D^2\frac{n}{k} \sum_{i=1}^k |\Wv^i\uv - \Wvalt^i\uv|^2 \\
&\leq B^2D^2\frac{n}{k} \sum_{i=1}^k \norm{\Wv^i - \Wvalt^i}^2 = B^2D^2\frac{n}{k} \normf{\Wv - \Wvalt}^2 .
\end{align*}
\end{proof}